\documentclass{article}

\usepackage[final]{neurips_2025}

\usepackage[utf8]{inputenc} 
\usepackage[T1]{fontenc}    
\usepackage[pdftex,colorlinks,linkcolor=blue,citecolor=blue,filecolor=blue,urlcolor=blue]{hyperref}       
\usepackage{url}            
\usepackage{booktabs}       
\usepackage{amsfonts}       
\usepackage{nicefrac}       
\usepackage{microtype}      
\usepackage{xcolor}         

\title{\textsc{SpEx}: A Spectral Approach to Explainable Clustering}

\author{%
  Tal Argov \\ 
  Tel Aviv University \\
  \texttt{talargov1@mail.tau.ac.il}
   \And
   Tal Wagner \\ 
  Tel Aviv University \\
  \texttt{talwag@tauex.tau.ac.il}
}

\usepackage{subfig}
\usepackage{wrapfig}
\usepackage{algorithm}
\usepackage{algorithmic}

\usepackage{amsmath}
\usepackage{amssymb}
\usepackage{mathtools}
\usepackage{amsthm}

\usepackage[capitalize,noabbrev]{cleveref}

\theoremstyle{plain}
\newtheorem{theorem}{Theorem}[section]

\newtheorem{fact}[theorem]{Fact}
\newtheorem{remark}[theorem]{Remark}
\newtheorem{lemma}[theorem]{Lemma}
\newtheorem{corollary}[theorem]{Corollary}
\theoremstyle{definition}
\newtheorem{definition}[theorem]{Definition}

\theoremstyle{remark}

\usepackage[textsize=tiny]{todonotes}

\newcommand{\spex}{\textsc{SpEx}}
\newcommand{\clique}{\textsc{SpEx}-Clique}
\newcommand{\knn}{\textsc{SpEx}-kNN}
\newcommand{\nc}{\ell}

\usepackage{amssymb,amsmath,amsthm,url}
\usepackage{tikz}
\usepackage{graphicx}
\usepackage{nicefrac}
\usepackage{aliascnt}

\newcommand{\norm}[1]{\lVert#1\rVert}

\def\E{\mathbb E}

\newcommand{\R}{\mathbb R}

\newcommand\argmin{\ensuremath{\mathrm{argmin}}}
\newcommand\vol{\ensuremath{\mathrm{vol}}}

\newcommand\numberthis{\addtocounter{equation}{1}\tag{\theequation}}

\newenvironment{CompactItemize}{
\begin{list}{\tiny$\bullet$}{%
\setlength{\leftmargin}{10pt}
\setlength{\itemindent}{0pt}
\setlength{\topsep}{-1pt}
\setlength{\itemsep}{0pt}
}}
{\end{list}}

\begin{document}
\setlength\intextsep{-0.2em}
\setlength{\columnsep}{1.5em}
\maketitle

\begin{abstract}
Explainable clustering by axis-aligned decision trees was introduced by~\cite{moshkovitz2020explainable} and has gained considerable interest. Prior work has focused on minimizing the price of explainability for specific clustering objectives, lacking a general method to fit an explanation tree to any given clustering, without restrictions. 
In this work, we propose a new and generic approach to explainable clustering, based on spectral graph partitioning. With it, we design an explainable clustering algorithm that can fit an explanation tree to any given non-explainable clustering, or directly to the dataset itself. Moreover, we show that prior algorithms can also be interpreted as graph partitioning, through a generalized framework due to~\cite{trevisan2013cheeger} wherein cuts are optimized in two graphs simultaneously. 
Our experiments show the favorable performance of our method compared to baselines on a range of datasets. 
\end{abstract}

\section{Introduction}
\label{sec:introduction}

\begin{wrapfigure}{R}{0.5\textwidth}
\begin{center}
    \subfloat{\includegraphics[width=0.5\linewidth]{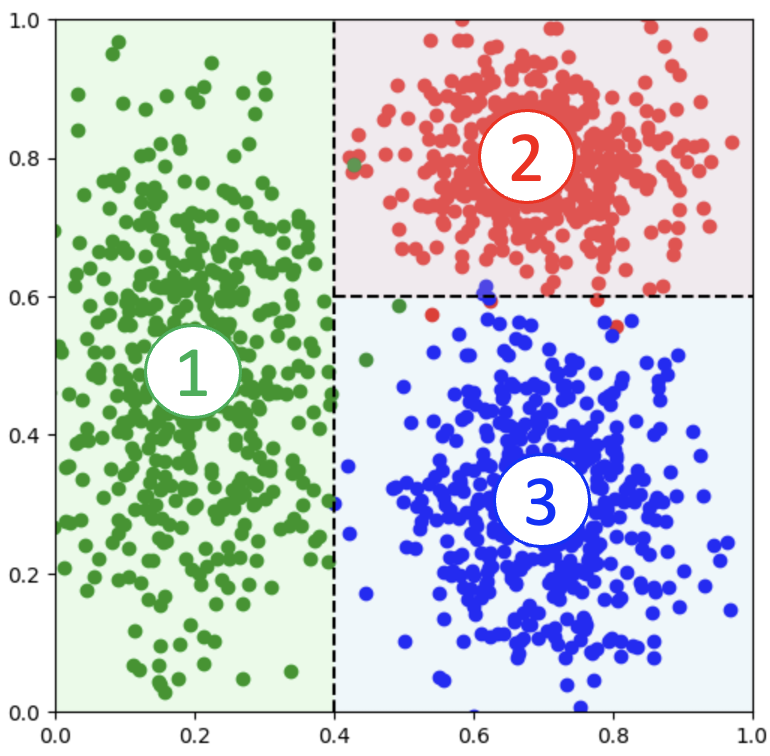}}
\hfil
    \subfloat{\includegraphics[width=0.5\linewidth]{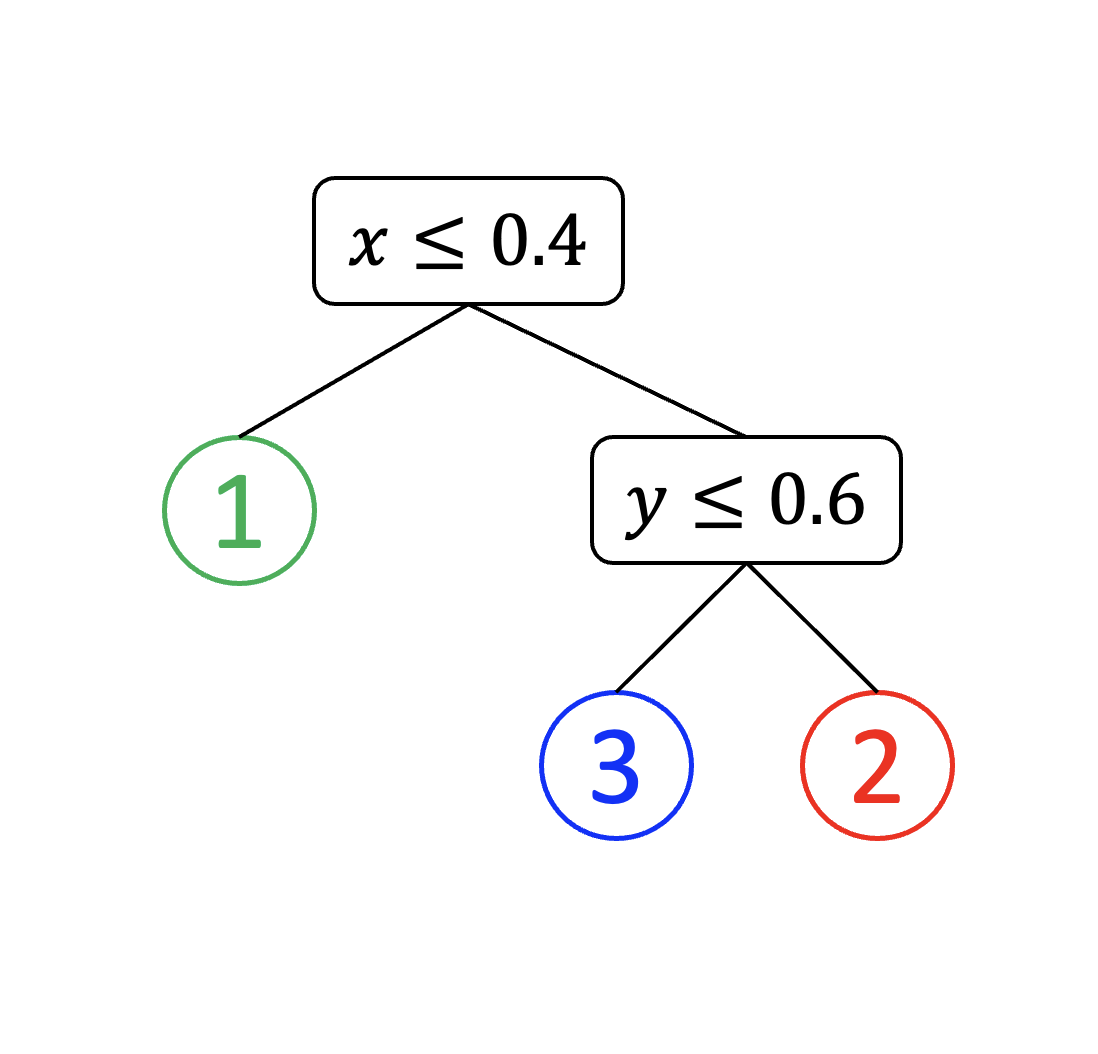}}
\caption{Illustration of explainable clustering. Clusters are generated from three gaussians. The dashed lines on the left and the decision tree on the right define the explainable clustering regions, with some points attributed to the wrong cluster.}
\label{fig:explainableclustering}
\end{center}
\end{wrapfigure}

As machine learning increasingly permeates daily life and forms the basis for consequential decision making in the real world, explaining its outputs in a manner that is interpretable to humans often becomes imperative. In a recent influential work, Moshkovitz et al.~\cite{moshkovitz2020explainable} proposed a model of explainability in clustering. In their model, a clustering of points in a feature space $\R^d$ is \emph{explainable} if it is described by a binary decision tree, where each internal node corresponds to thresholding the points along a single coordinate. Thus, the assignment of points to clusters can be described by a sequence of individual feature thresholds, arguably making it easy to explain and interpret. See~\Cref{fig:explainableclustering} for illustration.

Moshkovitz et al.~\cite{moshkovitz2020explainable} presented the Iterative Mistake Minimization (IMM) algorithm, which takes as input an already computed $k$-medians or $k$-means clustering of the data, called the \emph{reference clustering}, and ``rounds'' into an explainable clustering. It works by fitting the reference clustering with a decision tree that greedily minimizes wrong point-to-cluster assignments. 
They proved that the loss in clustering cost, called the \emph{price of explainability}, can be bounded as a function of $k$. This has led to surge of theoretical work on bounding the price of explainability, culminating in tight bounds for $k$-medians and nearly tight bounds for $k$-means~\cite{laber2021price,laber2023nearly,makarychev2021near,makarychev2022explainable,makarychev2024random,esfandiari2022almost,charikar2022near,gamlath2021nearly,gupta2023price,bandyapadhyay2023find,laber2024computational}.

This voluminous body of work has so far mostly focused on $k$-medians and $k$-means. 
These methods require the reference clustering to be endowed with centroids in order to work. This fails to capture widely used notions of clustering that do not produce centroids, like kernel $k$-means or spectral clustering~\cite{scholkopf1996nonlinear,dhillon2004kernel,von2007tutorial}, which one might wish to use as the reference clustering. \cite{fleissnerexplaining} recently took a first step in this direction, introducing the Kernel IMM algorithm, which extends IMM to handle kernel $k$-means as the reference clustering, for certain types of kernels.

Ideally, one would want a generic explainable clustering method that can be composed over any given reference clustering, irrespective of how it was computed, whether it has centroids, or any other constraint. 
The only such approach at present is through repurposing classical methods for supervised classification with decision trees, which predate the definition of explainable clustering by decades. 
CART~\cite{breiman1984classification,quinlan1986induction} is a method for decision tree learning over a labeled dataset,\footnote{CART is in fact a family of algorithms for Classification And Regression Trees (hence its name). The specific form of CART used in explainable clustering is detailed later, in~\Cref{subsec:CART}.} which outputs a tree with nodes corresponding to feature thresholds and leaves corresponding to class labels, with the aim of minimizing classification errors with respect to supervised labels. 
Thus, given any reference clustering,  its cluster assignments can be viewed as supervised labels, and CART can be used to round it into an explainable clustering.

\noindent\textbf{Limitations of current methods.} 
As discussed above, explainable clustering methods that require a centroid-based reference clustering, like $k$-means, are unable to handle cluster structures not captured by centroids. 
This was the motivation in~\cite{fleissnerexplaining}, who gave the example of the classical ``two moons'' dataset, depicted in~\Cref{fig:limitations} (left). Since $k$-means fails to capture the moon structure, any explainable clustering method built on it as the reference will fail as well. On the other hand, kernel $k$-means and spectral clustering capture the moons correctly, and can be rounded to the explainable rounding depicted in the figure, with a small number of errors. 

While the Kernel IMM algorithm does not require centroids, it is limited to kernel $k$-means with specific kernels, and varies by the specific kernel. Hence, it cannot be used as a generic method. Furthermore, the method is rather complex, and is evaluated in~\cite{fleissnerexplaining} on datasets of size only up to hundreds of points, leaving its scalability unclear. Our larger scale experiments indeed show that Kernel IMM becomes infeasible on bigger data.

The CART algorithm, as mentioned above, can be repurposed for explainable clustering even though it is not designed for this task. 
To demonstrate its drawbacks, \cite{moshkovitz2020explainable} gave a toy example where CART fails to find a simple error-free explainable clustering. It is shown in~\Cref{fig:limitations} (right), and we revisit this example in more detail in~\Cref{subsec:CART}. 
There have been some questions whether this failure mode is representative of behavior on real data, or merely a pathological case not encountered in practice.\footnote{See the peer review discussion for~\cite{fleissnerexplaining} at \url{https://openreview.net/forum?id=FAGtjl7HOw}.}  Preliminary small scale results in~\cite{fleissnerexplaining} implied that CART performs well empirically. However, our larger scale experiments will show that CART indeed suffers low performance on real data. 

The upshot is that despite a large body of work, there is no generic method for explainable clustering, which is oblivious to the type of reference clustering and robust across datasets. This is the gap that we address in this work.

\begin{figure}[t]
\vskip 0.2in
\begin{center}
    \subfloat{\includegraphics[width=0.22\linewidth]{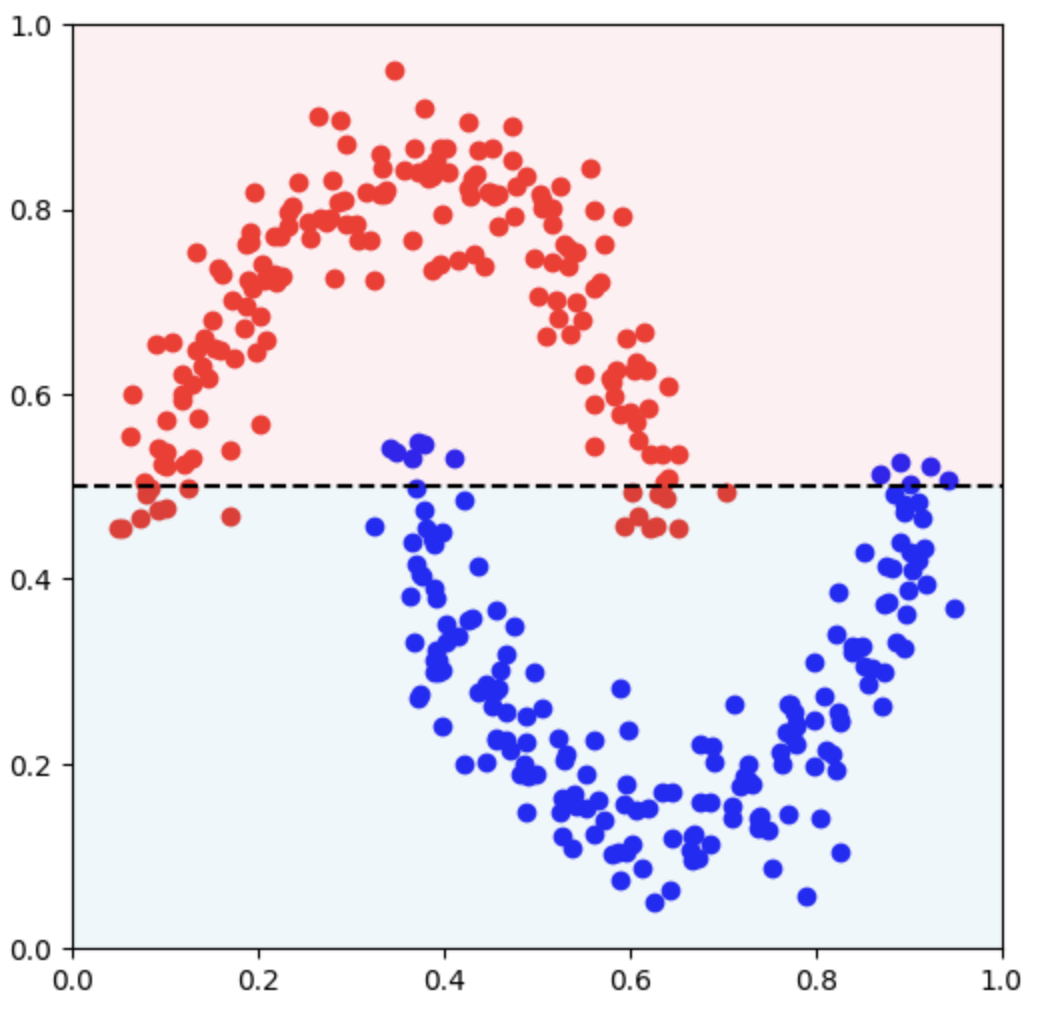}}
\hfil
    \subfloat{\includegraphics[width=0.22\linewidth]{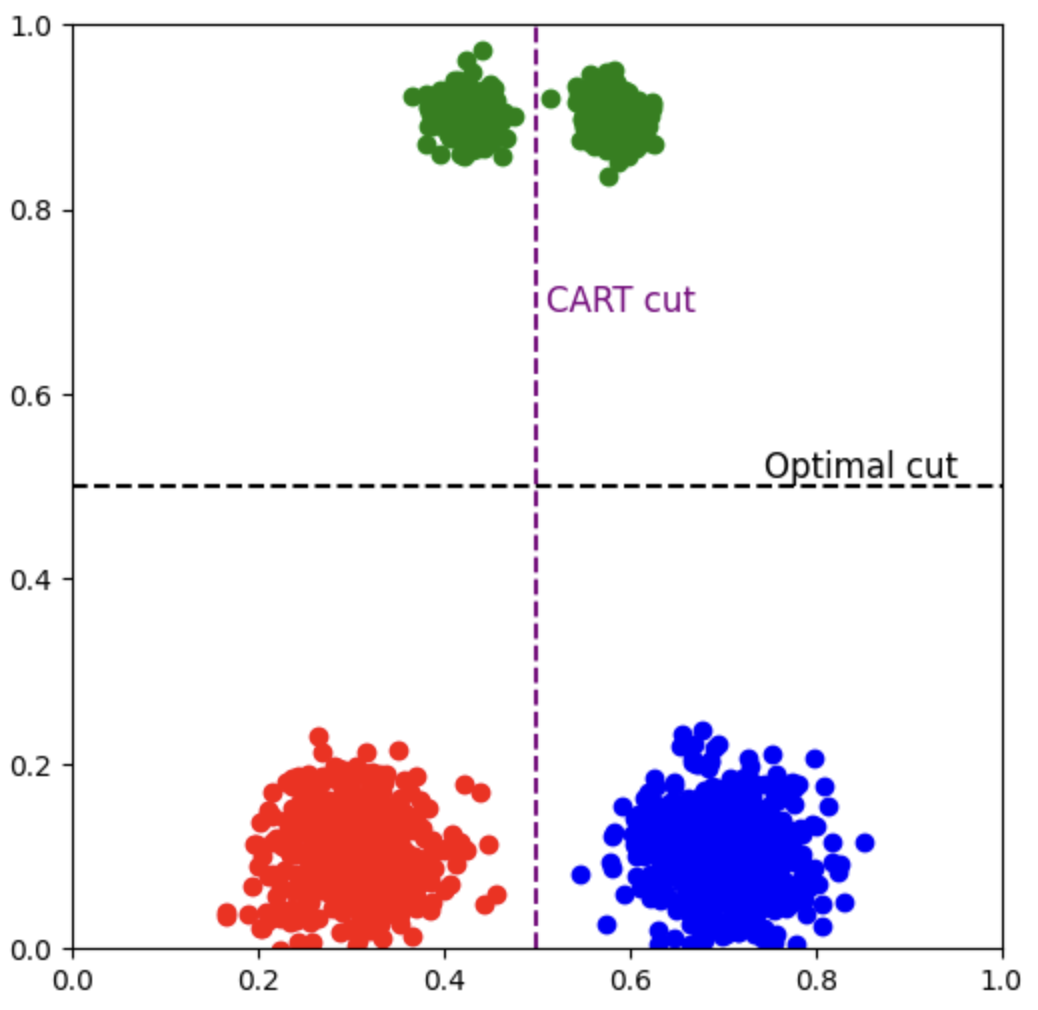}}
\caption{\textbf{Left:} The two-moons dataset admits an explainable clustering with small error. $k$-Means clustering will fail to capture the moons if used as reference clustering. Kernel $k$-means and spectral clustering capture the moons correctly, and can be rounded to the optimal explainable clustering. \textbf{Right:} A 3-way clustering example from~\cite{moshkovitz2020explainable}.  The horizontal cut leads to an error-free explainable clustering if chosen at the first step, but CART will select the error-heavy vertical cut first.}
\label{fig:limitations}
\end{center}
\vskip -0.2in
\end{figure}

\noindent\textbf{Our results.}
We introduce a new approach to explainable clustering, based on spectral graph techniques. Cheeger's inequality relates the eigenvalues of a graph to its optimal cut conductance, a fundamental graph partitioning objective that yields high-quality cuts. 
This useful fact is widely used in algorithms. 
A geometric consequence of it, stated in~\Cref{thm:geomcheeger}, relates coordinate cuts in $\R^d$ to cut conductances in the graph, provided that the graph usefully captures geometric relations between the data points in $\R^d$. This hints at a relevance to explainable clustering. 

We use this connection in two ways:
\begin{CompactItemize}
    \item\clique: We may describe any given reference clustering with a graph that contains a clique over each cluster. Applying the spectral result iteratively to this graph leads to an algorithm that can round any given reference clustering, without limitations, into an explainable clustering.
    \item\knn: We may also use the spectral result on a graph built directly on the dataset in $\R^d$, like the $k$-nearest neighbor graph. This leads to a ``reference-free'' algorithm, which computes an explainable clustering directly from the data, without using a reference clustering at all. 
\end{CompactItemize}


To gain more insight into existing methods and their relation to ours, we show they are captured by a generalized graph partitioning framework of non-uniform sparse cuts~\cite{trevisan2013cheeger}, allowing us to view them through a unified analytic lens.
Namely, prior methods can be seen as different choices of graphs to describe the reference clustering. 

Our experiments on a range of datasets and baseline shed light on the empirical performance of explainable clustering methods and showcases the advantage of our approach.

\subsection{Related Work}\label{sec:related}
Much work has focused on proving multiplicative bounds on the price of explainability for $k$-medians and $k$-means. In deterministic algorithms, IMM~\cite{moshkovitz2020explainable} achieved $O(k)$ for $k$-medians and $O(k^2)$ for $k$-means. \cite{esfandiari2022almost} improved the $k$-means bound to $O(k\log k)$, with an algorithm we will call EMN. 

Randomized methods have enabled better bounds. For $k$-medians, a sequence of works analyzed a natural randomized procedure and ultimately obtained a tight bound of $(1+o(1))\ln k$~\cite{makarychev2021near,makarychev2024random,esfandiari2022almost,gamlath2021nearly,gupta2023price}. For $k$-means, \cite{gupta2023price} achieved $O(k\ln\ln k)$, which is tight up to the $\ln\ln k$ term. 
Some works have also shown better bounds in the low-dimensional data regime~\cite{laber2021price,esfandiari2022almost,charikar2022near}, and studied the computational complexity of approximating the optimal explainable clustering~\cite{bandyapadhyay2023find,gupta2023price,laber2024computational}. 



Some works have used modified or extended definitions of tree-based explainable clustering \cite{fraiman2013interpretable,gabidolla2022optimal,laber2023shallow,shati2023optimal,cohen2023interpretable,upadhya2024neurcam,ohl2024kernel}.
In particular, \cite{gabidolla2022optimal} generalized the definition far beyond~\cite{moshkovitz2020explainable}, allowing oblique (hyperplane) cuts that involve multiple coordinates rather than just single coordinate cuts. This greatly increases the expressive capacity of the resulting clustering---in fact, \cite{gabidolla2022optimal} show it can exactly capture any reference $k$-means clustering, without any ``price of explainability''---though arguably at the price of rendering the clustering less explainable. 
While defining explainability in clustering remains an open-ended question, in this work we focus on the original explainability model of~\cite{moshkovitz2020explainable}. 

In the broader context of explainable machine learning, the model of \cite{moshkovitz2020explainable} is an example of an \emph{intrinsically explainability} method, where the given model (in this case, the reference clustering) is approximated by a different model (in this case, the explainable clustering tree) on which structural explainability constraints are imposed. This is in contrast with \emph{post-hoc explainability} methods, which aim to endow the given reference model with explanations without modifying it. 

\section{Spectral Explainable Clustering}
\label{sec:prelim}
We begin with some necessary background and notation on graph partitioning. Let $G(X,E,w)$ be an undirected graph. Given a strict non-empty subset $S\subset X$, let $e_G(S,X\setminus S)$ denote the sum of edge weights crossing between $S$ and $X\setminus S$, and let $\vol_G(S)$ denote the sum of weighted degrees of nodes in $S$ (called the~\emph{volume} of $S$).  

There are two standard notions for quantifying the outer connectivity of $S$ within $G$ (see~\cite{von2007tutorial}):
\begin{equation}\label{eq:setsparsity}
    \phi_G(S) = \frac{e_G(S,X\setminus S)}{|S|}
    \;\; ; \;\;
    \psi_G(S) = \frac{e_G(S,X\setminus S)}{\vol_G(S)} 
\end{equation}
$\phi_G(S)$ is called the \emph{sparsity} of $S$, while $\psi_G(S)$ is called the \emph{conductance} of $S$. Both notions have analogues for considering $S$ and its complement $X\setminus S$ as a two-way cut:
\begin{equation}\label{eq:cutsparsity}
    \Phi_G(S) = \frac{e_G(S,X\setminus S)}{\tfrac{1}{|X|}\cdot|S|\cdot|X\setminus S|} \;\; ; \;\;
    \Psi_G(S) = \frac{e_G(S,X\setminus S)}{\tfrac{1}{\vol_G(X)}\cdot\vol_G(S)\cdot\vol_G(X\setminus S)} .
\end{equation}
$\Phi_G(S)$ is sometimes called the~\emph{ratio cut} objective~\cite{wei1989towards}, while $\Psi_G(S)$ is called the~\emph{normalized cut} objective~\cite{shi2000normalized}. 
These objectives sometimes appear in the literature slightly differently, with $\min\{|S|,|X\setminus S|\}$ and $\min\{\vol_G(S),\vol_G(X\setminus S)\}$ as their respective denominators. These are the same as~\cref{eq:cutsparsity} up to a factor of $2$, since $\tfrac12\min\{s,n-s\} \leq \tfrac{1}{n}\cdot s\cdot (n-s) \leq \min\{s, n-s\}$ for all $0<s<n$, which applies here with $n=|X|$, $s=|S|$ in $\Phi_G(S)$, and with $n=\vol_G(X)$, $s=\vol_G(S)$ in $\Psi_G(S)$.

\subsection{The Spectral Approach} \label{subsec:theory}
Our approach is grounded in the following theorem, which is a generalization of a theorem given in \cite{donglearning}, based on a framework developed by \cite{andoni2018data,andoni2018holder}, in the context of nearest neighbor search. 
It relates coordinate cuts in $\R^d$ to the Cheeger inequality \cite{alon1985lambda1,chung1997spectral}, a fundamental result in spectral graph theory with a myriad of algorithmic implications.
While they stated the theorem specifically for the nearest neighbor graph of a point set in $\R^d$, it holds for general weighted graphs over points in $\R^d$, as we now state. 
We provide a proof in \Cref{sec:proofs}.

\begin{definition}
   Let $X\subset\R^d$. The \emph{coordinate cut} given by coordinate $j\in\{1,\ldots,d\}$ and threshold $\tau\in\R$ is
   $S_{j,\tau}(X) := \{x\in X: x_j \leq \tau\}$.
The cut is~\emph{valid} if $S_{j,\tau}(X)\neq\emptyset$ and $S_{j,\tau}(X)\neq X$. 
\end{definition}

\begin{theorem}\label{thm:geomcheeger}
    Let $X\subset\R^d$ be a set of points, where $x\in X$ has coordinates $x=(x_1,\ldots,x_d)$. Let $G(X,E,w)$ be a graph with vertex set $X$. Consider two distributions over pairs of points in $X$:
    \begin{CompactItemize}
        \item $\mathcal D_{\mathrm{adj}}$ is the distribution over adjacent pairs in $G$, where a pair $x,y\in E$ is sampled with probability proportional to the edge weight between them.
        \item $\mathcal D_{\mathrm{all}}$ is the distribution over all pairs $x,y\in X$, where $x$ and $y$ are sampled independently, each with probability proportional to its weighted degree in $G$. 
    \end{CompactItemize} 
    Then, there is a valid coordinate cut $j,\tau$
    such that
    \[ \Psi_G(S_{j,\tau}(X)) \leq \sqrt{\frac{\E_{x,y\sim\mathcal D_{\mathrm{adj}}}\norm{x-y}_2^2}{\E_{x,y\sim\mathcal D_{\mathrm{all}}}\norm{x-y}_2^2}} . \]
\end{theorem}

Intuitively, the theorem draws a connection between combinatorial graph cuts and geometric properties of the set of points: if the graph nodes graph are associated with embeddings in $\R^d$, then there is a ``good quality'' graph cut, whose conductance is upper-bounded in terms of the squared Euclidean distances between the node embeddings. Furthermore, that cut is a coordinate cut. 

The conductance bound is governed by the ratio of expected squared distances in $X$ according to two distributions over pairs of points: the numerator samples pairs of points adjacent in $G$, while the denominator samples any pair of points. Thus, the ratio (and hence the cut conductance) is smaller when $G$ captures a geometrically meaningful structure in the dataset, wherein adjacent pairs of points are expected to be nearer to each other than general pairs of points. 

This implies an approach to explainable clustering. 
Given a reference clustering, we can describe it with a suitable graph $G$, and iteratively look for the coordinate cut with minimum conductance in each node of the decision tree. 
If the reference clustering is of good quality, in the sense that a pair of points are expected to be nearer to each other if they are clustered together, then~\Cref{thm:geomcheeger} guarantees the existence of a low-conductance cut. 
At the same time, since the nearness of points in the same cluster is a ``soft'' requirement (it only needs to hold in expectation), it renders the theorem robust to various types of clustering, allowing for unconstrained cluster shapes, outliers, etc. 

Another possibility that arises is to dispense with the reference clustering, and construct a graph that captures nearness/farness directly from the dataset. 
In~\Cref{sec:tree} we describe the explainable tree construction with a general graph, and in~\Cref{sec:graph} we discuss graph selection.

\newcommand{\INDENT}{\hspace{1em}}
\begin{wrapfigure}{R}{0.5\textwidth}
\begin{minipage}{0.5\textwidth}
\begin{algorithm}[H]
\caption{\;\;\spex} 
\label{alg:main}
\textbf{input:} Dataset $X\subset\R^d$, graph $G(X,E,w)$, target number of clusters $\nc$

\textbf{output:} Decision tree $T$ where every internal node is associated with a coordinate $j$ and threshold $\tau$
\smallskip{\hrule height.2pt}\smallskip
\textsc{BuildTree$(X,G,\nc)$:}
\begin{algorithmic}
\vspace{-1em}
   \STATE $T\leftarrow$ initialize a tree with a single node $v$
    \STATE $j,\tau\leftarrow\argmin_{j,\tau}\textsc{CutScore}(X,j,\tau)$
    \STATE $Q\leftarrow$ initialize a maximum priority for the tree leaves, with priorities given by \textsc{LeafScore}
   \STATE $Q.push(v,X,j,\tau)$ 
   \STATE {\bfseries while} $T$ has less than $\nc$ leaves {\bfseries do}
   \STATE\INDENT $v,X_v,j_v,\tau_v\leftarrow Q.pop()$ 
   \STATE\INDENT Associate $v$ with the cut $j,\tau$\\
   \STATE\INDENT Split $v$ into two new leaves $v_L,v_R$\\
   \STATE\INDENT $X_{v_L}\leftarrow S_{j_v,\tau_v}(X_v)$
   \STATE\INDENT $X_{v_L}\leftarrow X_v\setminus X_{v_L}$
   \STATE\INDENT $j_L,\tau_L\leftarrow\argmin_{j,\tau}\textsc{CutScore}(X_{v_L},j,\tau)$
   \STATE\INDENT $j_R,\tau_R\leftarrow\argmin_{j,\tau}\textsc{CutScore}(X_{v_R},j,\tau)$
   \STATE\INDENT $Q.push(v_L,X_{v_L},j_L,\tau_L)$ 
   \STATE\INDENT $Q.push(v_R,X_{v_R},j_R,\tau_R)$ 
   \STATE{\bfseries return} $T$
\end{algorithmic}
\smallskip{\hrule height.2pt}\smallskip
$\textsc{CutScore}(X',j,\tau):$ 
\begin{algorithmic}
      \STATE \textbf{if} $S_{j,\tau}(X)=\emptyset$ or $S_{j,\tau}(X)= X$ \textbf{then}
   \STATE\INDENT {\bfseries return} $\infty$
   \STATE {\bfseries return} $\psi_G(S_{j,\tau}(X')) + \psi_G(X_v\setminus S_{j,\tau}(X'))$
\end{algorithmic}
\smallskip{\hrule height.2pt}\smallskip
$\textsc{LeafScore}(v, X',j,\tau):$
\begin{algorithmic}
   \STATE {\bfseries return} $\psi_G(X') - \textsc{CutScore}(X',j,\tau)$
\end{algorithmic}
\end{algorithm}
\end{minipage}
\vspace{1em}
\end{wrapfigure}

\subsection{Iterative Tree Construction}\label{sec:tree}
Given a graph $G(X,E,w)$ over the dataset $X\subset\R^d$, our task is to construct an explainable clustering decision tree $T$. Let $\nc>0$ be the desired number of leaves. 

An explainable clustering tree is a decision tree in which every internal node $v$ is associated with a coordinate $j_v\in\{1,\ldots,d\}$ and a threshold $\tau_v\in\R$, inducing the condition $x_{j_v}\leq\tau_v$ given a point $x\in\R$. The coordinate thresholds associate a subset $X_v\subset X$ with each node $v$: the root is associated with $X$, 
and every non-root $v$ with parent $u$ is associated with $X_v=S_{j_u,\tau_u}(X_u)$ if $v$ is the left child of $u$, or $X_v=X_u\setminus S_{j_u,\tau_u}(X_u)$ if it $v$ the right child of $u$.

Let $\mathcal{L}(T)$ denote the set of leaves in $T$. They induce partition of $X$ into clusters, $X=\cup_{v\in\mathcal{L}(T)}X_v$. We measure the quality of $T$ by a generalization of the normalized cut objective to multi-way partitions~\cite{von2007tutorial}:
\begin{equation}\label{eq:treeobj}
    \bar\Psi_G(T) = \sum_{v\in\mathcal{L}(T)}\psi_G(X_v) .
\end{equation}
The smaller $\bar\Psi_G(T)$, the better the partition induced by its leaves. 
Note this generalizes the two-way cut objective $\Psi_G(S)$ (\cref{eq:cutsparsity}), since $\Psi_G(S)=\psi_G(S)+\psi_G(X\setminus S)$. 

To build the tree, we start with $T$ containing only a root. As long as $T$ does not yet have the requisite number of leaves $\nc$, we choose the leaf $v$ to split next to greedily minimize $\bar\Psi_G(T)$. Splitting a leaf $v$ with a cut $S\subset X_v$ 
replaces the summand $\psi_G(X_v)$ in~\cref{eq:treeobj} with $\psi_G(S) + \psi_G(X_v\setminus S)$, and thus $v$ is chosen to maximize the reduction in $\bar\Psi_T(G)$ its split would yield, which is
\begin{equation}\label{eq:leafscore}
    \psi_G(X_v) - \min_{S}\left(\psi_G(S) + \psi_G(X_v\setminus S)\right),
\end{equation}
where $S$ ranges over all valid coordinate cuts. We associate $v$ with the $j_v,\tau_v$ corresponding to the minimizer $S$, split $v$ into two new leaves along this cut, and repeat.
See~\Cref{alg:main}.

Note that~\Cref{alg:main} has the flexibility to produce a tree $T$ with any desired number of leaves, regardless of the number of clusters in the reference clustering. In contrast, centroid-based methods like IMM, EMN and Kernel IMM are bound to produce the same number of leaves in $T$ as the number of centroids in the reference $k$-median or (kernel) $k$-means clustering, and increasing the number of leaves requires separate techniques~\cite{frost2020exkmc,makarychev2022explainable,fleissnerexplaining}. 
Increasing the number of leaves is helpful for attaining a smaller price of explainability due to the increased expressivity of $T$, albeit at the expense of being less explainable due to its bigger size. We discuss this further in Appendix \ref{app:experiments}. 
Our main evaluation will focus on producing trees with the same number of leaves as clusters in the reference clustering.



\subsection{\clique~and~\knn}\label{sec:graph}
To capture a given reference clustering with a graph, there are several natural choices:
\begin{CompactItemize}
    \item\emph{Clique graph:} points are adjacent if and only if they are in the same cluster (thus, every cluster becomes a clique).
    \item\emph{Star graph:} if the reference clustering is endowed with centroids, each point can be made adjacent to its cluster centroid (thus, every cluster becomes a star). 
    \item\emph{Independent set (IS) graph:} points are adjacent if and only if they are~\emph{not} in the same cluster (thus, every cluster becomes an independent set). This is the complement of the clique graph. Here, of course, one would wish to maximize rather than minimize the cut objective.  
\end{CompactItemize}

The clique graph may seem like the most natural choice, and this is indeed the one we make. 
Perhaps surprisingly, as we will show in~\Cref{sec:nonuniform}, when previous methods (IMM, EMN, CART) are interpreted through this graphical lens, they turn out to correspond to either the star graph or the IS graph. This directly relates to their limitations, like requiring centroids (in the case of the star graph) or failing the toy example from~\Cref{fig:limitations} (in the case of the IS graph). 

Given any reference clustering, using the clique graph as $G$ in~\Cref{alg:main} yields the algorithm we call \clique. We also consider an variant that uses no reference clustering, by constructing the nearest neighbor graph directly on the points in $X$. This yields the algorithm we call~\knn.


\subsection{Computational Efficiency}
\spex~as well as the baselines we consider share the following high-level structure. In each tree node $u$ with $n_u$ points, for each coordinate, they sort the points by that coordinate (time $O(n_u\log n_u)$) and perform a sweep-line procedure that iterates over the $n_u-1$ prefix/suffix cuts by moving nodes from the suffix to the prefix one at a time. As each node is moved, the cut score is updated accordingly (let $S$ denote the time it takes to update), and eventually the cut with the overall optimal score is selected. Repeating this for all coordinates takes time $O(dn_u(\log n_u + S))$. In each tree level, the sum $\sum_un_u$ is $n$ for \spex~and CART and $n+k\leq 2n$ for IMM and EMN, therefore the time per level is $O(dn(\log n + S))$. Summing over up to $k-1$ levels in the tree, the total time is $O(kdn(\log n + S))$.

The algorithms may differ on the time $S$ it takes update cut scores during sweep-line. In \clique, we need not store the entire clique graph; rather, we only store point-cluster assignments ($O(n)$ memory). During sweep-line, we maintain two cluster histograms for the prefix and the suffix ($O(k)$ memory), from which the cut score can be updated in $S=O(1)$ time. In \knn, letting $\kappa$ denote the number of neighbors per node (note that this is a different parameter than the number of clusters $k$), we can store the kNN graph in $O(n\kappa)$ memory and update cut scores in $S=O(\kappa)$ time.

\section{Lens: Non-Uniform Sparse Cuts}\label{sec:nonuniform}
While our spectral approach to explainable clustering may seem rather different from previous methods, in this section we show a generalized graph partitioning framework that captures them in a unified way. It is based on \emph{non-uniform sparse cuts} as defined by Trevisan~\cite{trevisan2013cheeger}, where cuts are optimized simultaneously in two graphs that share the same set of nodes.

Trevisan~\cite{trevisan2013cheeger} defined the Non-Uniform Sparsest Cut problem as follows. Let $G(X,E_G,w_G)$ and $H(X,E_H,w_H)$ be two graphs on the same set of nodes $X$. 
The $(G,H)$-sparsity of a cut $(S,X\setminus S)$ is defined as the (normalized) ratio of edges cut in $G$ to edges cut in $H$:
\[
  \Psi_{G,H}(S) = \frac{\tfrac{1}{\vol_G(X)}\cdot e_G(S,X\setminus S)}{\tfrac{1}{\vol_H(X)}\cdot e_H(S,X\setminus S)} . 
\]
The goal in the Non-Uniform Sparsest Cut problem is to find the cut with the smallest $\Psi_{G,H}(S)$. 

In~\cite{trevisan2013cheeger} it is observed that this generalized graph partitioning problem captures several classical problems defined on a single graph $G$ as special cases:
\begin{CompactItemize}
    \item The classical (``uniform'') Sparsest Cut problem, of minimizing $\Phi_G(S)$ from~\cref{eq:cutsparsity}, is the case of $\Psi_{G,H}(S)$ where $H$ is an unweighted clique over $X$. 
    \item The Normalized Cut problem, of minimizing $\Psi_G(S)$ from~\cref{eq:cutsparsity}, is the case of $\Psi_{G,H}(S)$ where $H$ is the~\emph{$G$-degree weighted clique} over $X$, in which the edge weight in $H$ between every pair $x,y\in X$ is the product of their weighted degrees in $G$.
    \item The Minimum $st$-Cut problem is the case of $\Psi_{G,H}(S)$ where $H$ contains a single edge between a ``distinguished'' pair $s,t\in X$. 
\end{CompactItemize}


Here, we further observe that this framework is useful in capturing prior methods for explainable clustering (or close variants of those methods), as they in fact produce coordinate cuts that minimize $\Psi_{G,H}(S)$ with particular choices of graphs $G$ and $H$, even though neither method is originally given in terms of graphical sparse cut terms. 

\begin{CompactItemize}
    \item IMM~\cite{moshkovitz2020explainable} corresponds to $G$ being the star graph over a given reference clustering endowed with centroids, and $H$ containing a single edge that connects a pair of diametrical (furthest) centroids. 
    \item EMN~\cite{esfandiari2022almost} also corresponds to $G$ being the star graph, but with $H$ being an unweighted clique over the $k$ cluster centroids. 
    \item CART closely corresponds to $H$ being the independent set (IS) graph over the reference clustering (see~\Cref{sec:graph}), and $G$ being the $H$-degree weighted clique.
\end{CompactItemize}
We now discuss each algorithm in turn, and highlight useful implications of this graph-theortic lens. 

\subsection{IMM as Non-Uniform Sparse Cut}\label{sec:imm}
Given a reference clustering $\mathcal C=(C_1,\ldots,C_k)$ endowed with cluster centroids $M=\{\mu^{(1)},\ldots,\mu^{(k)}\}$, the IMM algorithm builds a decision tree in which every node $u$ is associated with a subset $X_u\subset X\cup M$ of points and centroids. 
In each node, it chooses the coordinate cut that minimizes the number of ``mistakes'', i.e., of points placed at a different side of the cut than their centroid, while separating at least one pair of centroids. Formally, for $x\in X$, let $\mu(x)$ denote the centroid of the cluster containing $x$. The cut $S\subset X_u$ in a node $u$ is chosen by IMM to minimize,
\[
  \mathrm{mis}_u(S) = |\{x\in S: \mu(x)\in X_u\setminus S\} \cup \{x\in X_u\setminus S: \mu(x)\in S\}|,
\]
among all coordinate cuts $S$ that satisfy $\mu^{(i)}\in S$ and $\mu^{(j)}\in X\setminus S$ for at least one pair $\mu^{(i)},\mu^{(j)}\in M\cap X_u$. The construction terminates when each leaf in the tree contains exactly one centroid. 


To cast this as non-uniform sparse cut, consider a slightly modified variant of IMM. For each node $u$, Let $\mu',\mu''\in X_u$ be a pair of centroids at maximal distance among the centroids in $M\cap X_u$. Choose the coordinate cut $S\subset X_u$ that minimizes $\mathrm{mis}_u(S)$ among those that separate $\mu',\mu''$. 
This is a subset of the cuts in the original IMM. 
%
For this variant, let $G$ be the \emph{\textbf{star graph}} over the reference clustering, and $G_u=G[X_u]$ be its restriction to the subset $X_u\subset X\cup M$ associated with each tree node $u$. Then $\mathrm{mis}_u(S) = e_{G_u}(S,X_u\setminus S)$, since each cut edge marks a point separated from its centroid. 
Let $H_u$ be the \emph{\textbf{single edge graph}} over $X_u$ that contains only an edge connecting $\mu'$ and $\mu''$. 
Then we have, $e_{H_u}(S,X_u\setminus S)=1$ if the cut $(S,X_u\setminus S)$ separates $\mu',\mu''$, and $e_{H_u}(S,X_u\setminus S)=0$ otherwise. 
Thus, minimizing $\mathrm{mis}_u(S)$ subject to a cut separating $\mu',\mu''$ is equivalent to minimizing the ratio $e_{G_u}(S,X_u\setminus S)/e_{H_u}(S,X_u\setminus S)$. This is equal (up to normalization) to the non-uniform cut sparsity, $\Psi_{G_u,H_u}(S)$. 
Even though this modified IMM considers less cuts than the original IMM, it attains the same price of explainability for $k$-medians, by a different proof based on graph partitioning, that avoids the intricate potential function analysis in \cite{moshkovitz2020explainable}. We show this in \Cref{sec:immproof}.

\subsection{EMN as Non-Uniform Sparse Cut}
The EMN algorithm~\cite{esfandiari2022almost} is an improvement of IMM by the following modification: in every tree node $u$, instead of minimizing $\mathrm{mis}_u(S)$, it chooses the threshold coordinate cut $S\subset X_u$ that minimizes the ratio 
$\mathrm{mis}_u(S)/f_u(S)$, where~\cite{esfandiari2022almost} define 
$f_u(S) = \min\{|S\cap M|, |(X_u\setminus S)\cap M|\}$. 
As with IMM, letting $G_u$ be the \emph{\textbf{star graph}} over the reference clustering when restricted to $X_u$, we have $\mathrm{mis}_u(S)=e_{G_u}(S,X_u\setminus S)$. Let $H_u$ be the graph with vertex set $X_u$ whose edges form an \emph{\textbf{unweighted clique}} over $M\cap X_u$ (the rest of the vertices in $X_u$ are isolated in $H_u$). Then, 
\[ e_{H_u}(S,X_u\setminus S)=|S\cap M|\cdot|(X_u\setminus S)\cap M| . \]
Recall from \Cref{sec:prelim} that
$\tfrac12 f_u(S)\leq\tfrac{1}{|M\cap X_u|}\cdot e_{H_u}(S,X_u\setminus S) \leq f_u(S)$.
Thus, minimizing the ratio $\mathrm{mis}_u(S)/f_u(S)$ is equivalent, up to a factor of $2$, to minimizing the ratio $e_{G_u}(S,X_u\setminus S)/e_{H_u}(S,X_u\setminus S)$. This is equivalent to minimizing the non-uniform cut sparsity $\Psi_{G_u,H_u}(S)$.

\subsection{CART as Non-Uniform Sparse Cut} \label{subsec:CART}
We first review the CART algorithm in the form used in explainable clustering. The~\emph{Gini impurity} of a distribution $(p_1,\ldots,p_k)$ over $k$ elements is defined as $\upsilon(p)=\sum_{i=1}^kp_i(1-p_i)$. Given a point set $X\subset\R^d$, a reference clustering $\mathcal C=\{C_1,\ldots,C_k)$, and a subset $S\subset X$, we may define a distribution $p_{\mathcal C}(S)$ over the clusters as
$p_{\mathcal C}(S) = \left(\frac{|S\cap C_1|}{|S|},\frac{|S\cap C_2|}{|S|},\ldots,\frac{|S\cap C_k|}{|S|}\right)$,
which corresponds to sampling a uniformly random point from $S$ and returning its cluster ID. The impurity of $S$ is defined as $\upsilon_{\mathcal C}(S)=\upsilon(p_{\mathcal C}(S))$,
and the impurity of $(S,X\setminus S)$ as a two-way cut as,
\[ \Upsilon_{\mathcal C}(S,X\setminus S) = \frac{|S|}{|X|}\cdot\upsilon_{\mathcal C}(S) + \frac{|X\setminus S|}{|X|}\cdot\upsilon_{\mathcal C}(X\setminus S) . \]
The CART algorithm builds a decision tree starting with a root node associated with all of $X$. It then iteratively chooses a leaf node whose associated subset $X'$ maximizes 
$\mathrm{score}(X')=\upsilon_{\mathcal C}(X')-\min_{S\subset X'}\Upsilon_{\mathcal C}(S,X'\setminus S)$,
where $S$ ranges over all coordinates threshold cuts, and splits that leaf across that cut. 
The score captures the impurity reduction gained by splitting $X'$ along its best cut. 
Note how this is analogous to~\Cref{eq:leafscore} in~\spex.

Now, we show that a small modification to CART is equivalent to a case of non-uniform sparsest cut. Let $H$ be the \emph{independent set} graph over the reference clustering $\mathcal C$ (see \Cref{sec:graph}). 
For a subset $S\subset X$, Let $\alpha_H(S)$ denote the sum of edge weights with both endpoints in $S$. 
We can now interpret the impurity $\upsilon_{\mathcal C}(S)$ of a subset $S\subset X$ as,
\[
\upsilon_{\mathcal C}(S) = 
\sum_{i=1}^k\frac{|S\cap C_i|}{|S|}\left(1 - \frac{|S\cap C_i|}{|S|}\right) 
    =  \frac{1}{|S|^2}\sum_{i=1}^k|S\cap C_i|\left(|S| - |S\cap C_i|\right) \\
    = \frac{2\cdot\alpha_H(S)}{|S|^2} .
\]
Thus, the two-way cut impurity becomes
\begin{equation}\label{eq:cutimpurity}
    \Upsilon_{\mathcal C}(S,X\setminus S) = \frac{2}{|X|}\left(\frac{\alpha_H(S)}{|S|} + \frac{\alpha_H(X\setminus S)}{|X\setminus S|}\right) .
\end{equation}
To proceed, we make the modification of replacing set cardinalities with their volumes in $H$, as in the transition from sparsity to conductance in~\Cref{eq:setsparsity,eq:cutsparsity}. We get the modified two-way cut impurity, instead of \Cref{eq:cutimpurity}:
\[
    \widetilde\Upsilon_{\mathcal C}(S,X\setminus S) = \frac{2}{\vol_H(X)}\left(\frac{\alpha_H(S)}{\vol_H(S)} + \frac{\alpha_H(X\setminus S)}{\vol_H(X\setminus S)}\right) .
\]
This is precisely (up to normalization terms) the $Nassoc$ graph cut objective defined in~\cite[eq.~(3)]{shi2000normalized}.\footnote{In the notation of~\cite{shi2000normalized}, $\alpha_H(S)=assoc(S,S)$ and $\vol_H(S)=assoc(S,X)$.} As they observed, it is directly related to the normalized cut objective, i.e., to the cut conductance: 
\[
  \Psi_H(S) = 2 - \tfrac{\vol_H(X)}{2}\cdot\widetilde\Upsilon_{\mathcal C}(S,X\setminus S) .  
\]
Thus, finding a coordinate cut $S$ that minimizes our modified impurity $\widetilde\Upsilon$ is equivalent to~\emph{maximizing} the usual cut conductance $\Psi_H(S)$ in $H$. This, in turn, is equivalent to maximizing the non-uniform cut sparsity $\Psi_{H,G}(S)$, where $H$ is the \emph{\textbf{independent set}} graph and $G$ is the $H$-\emph{\textbf{degree weighted clique}} over $X$, since (as observed in \cite{trevisan2013cheeger} and mentioned in the beginning of this section), for this choice of graphs it holds that $\Psi_{H,G}(S)=\Psi_H(S)$, up to normalization. Finally, since $\Psi_{G,H}(S)=1/\Psi_{H,G}(S)$ by definition, this is equivalent \emph{minimizing} the non-uniform cut sparsity $\Psi_{H,G}(S)$. 

This point of view clarifies the failure of CART in~\Cref{fig:limitations}. The IS graph over a reference clustering incentivizes cuts that separate pairs of points residing in different clusters, but has no incentive to~\emph{not} separate points residing in the same cluster, since they are not connected with any edges and hence have no effect on the cut value. Since CART maximizes cut conductance, it opts to cut the larger number of edges between the bottom clusters (vertical cut) over the smaller number of edges between the top cluster and each of the bottom clusters (horizontal cut), which is incompatible with clustering.

\section{Experimental Evaluation} 
\label{sec:experiments}


\begin{table}[t]
\centering
\caption{Datasets. $^*$Training set only. $^\dagger$Embedded with CLIP \cite{Radford2021LearningTV}. $^\ddagger$Embedded with SBERT \cite{reimers-2019-sentence-bert}.}
\label{tab:dataset-info}
\vspace{0.1in}
\centering
\small
\setlength{\tabcolsep}{2.2pt} 
\begin{tabular}{l c c c c c c c c}
\toprule
 & CIFAR-10$^*$ & 20Newsgroups$^*$ & MNIST$^*$ & Caltech 101 & Beans & Breast Cancer & Ecoli & Iris \\
\midrule
\# Points & 50,000 & 11,314 & 60,000 & 8677 & 13,611 & 569 & 336 & 150 \\
\# Dimensions & 512$^\dagger$ & 768$^\ddagger$ & 512$^\dagger$ & 512$^\dagger$ & 16 & 30 & 7 & 4 \\
\# Classes & 10 & 20 & 10 & 101 & 7 & 2 & 8 & 3 \\
Reference & \cite{Krizhevsky2009LearningML} & \cite{bay2000uci} &  \cite{LeCun1998GradientbasedLA} & \cite{FeiFei2006OneshotLO} & \cite{Koklu2020MulticlassCO} & \cite{Street1993NuclearFE} & \cite{Horton1996APC} & \cite{Fisher1936THEUO} \\
\bottomrule
\end{tabular}
\end{table}

\begin{figure*}[ht]
\begin{center}

    \subfloat{\includegraphics[width=0.25\linewidth]{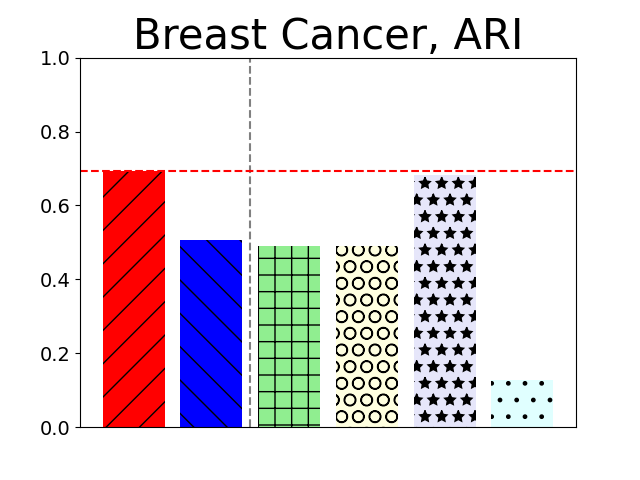}}\hspace{-0.3in}
\hfil
    \subfloat{\includegraphics[width=0.25\linewidth]{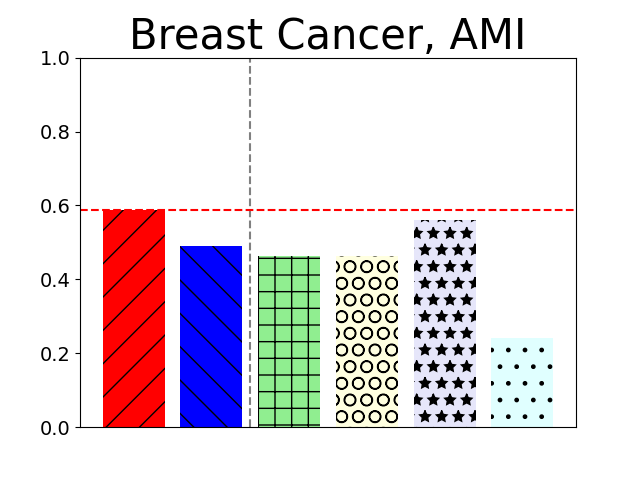}}
\hfil
    \subfloat{\includegraphics[width=0.25\linewidth]{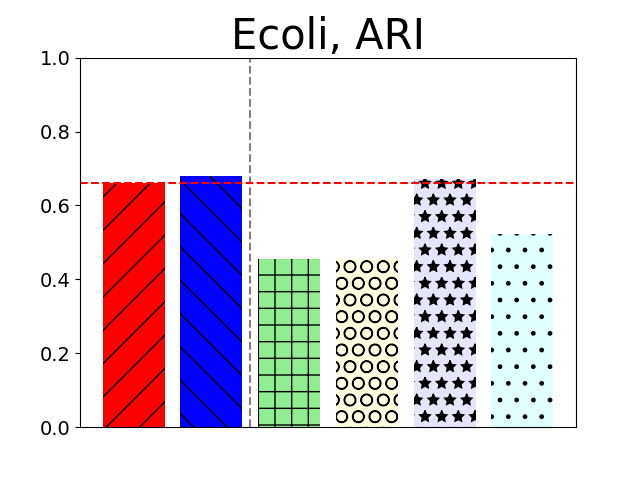}}\hspace{-0.3in}
\hfil
    \subfloat{\includegraphics[width=0.25\linewidth]{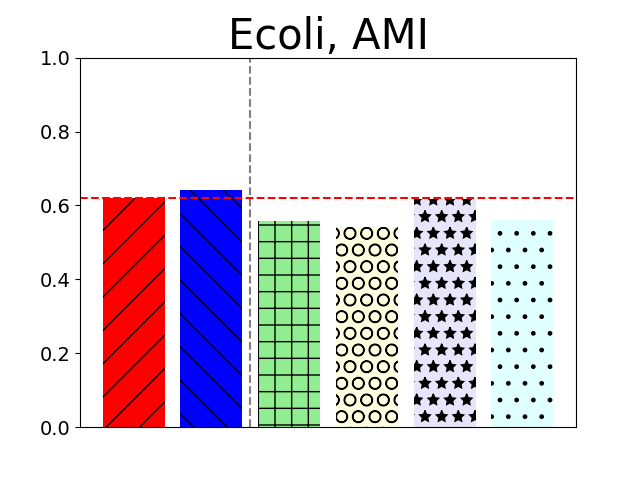}}

\vspace{-1.2em}
    \subfloat{\includegraphics[width=0.25\linewidth]{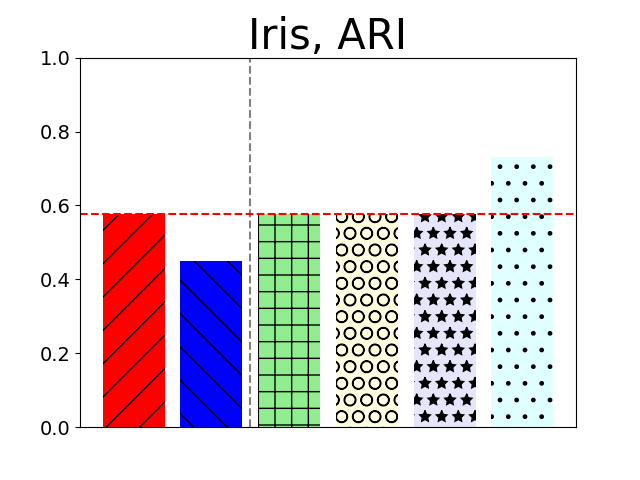}}\hspace{-0.3in}
\hfil
    \subfloat{\includegraphics[width=0.25\linewidth]{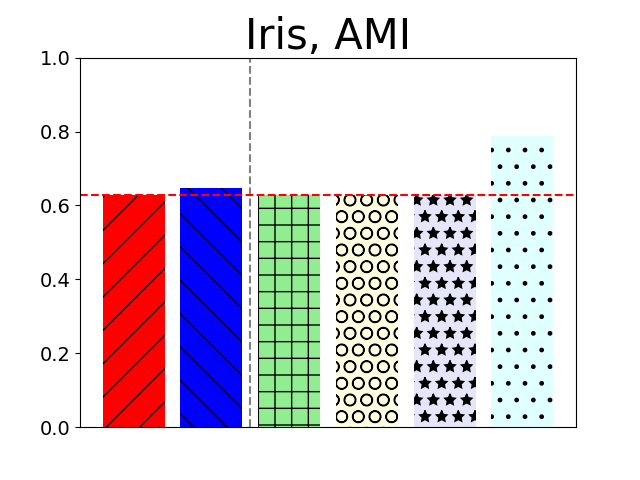}}
\hfil
    \subfloat{\includegraphics[width=0.25\linewidth]{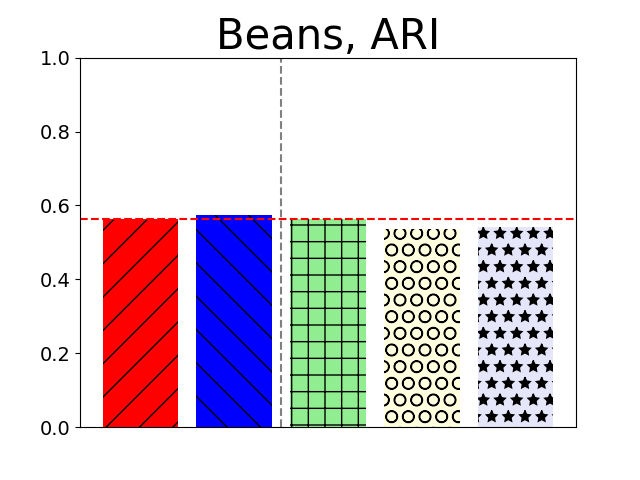}}\hspace{-0.3in}
\hfil
    \subfloat{\includegraphics[width=0.25\linewidth]{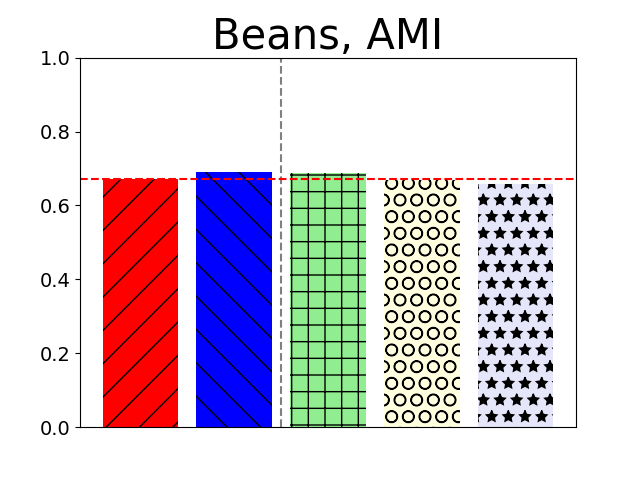}}

\vspace{-1.2em}
    \subfloat{\includegraphics[width=0.25\linewidth]{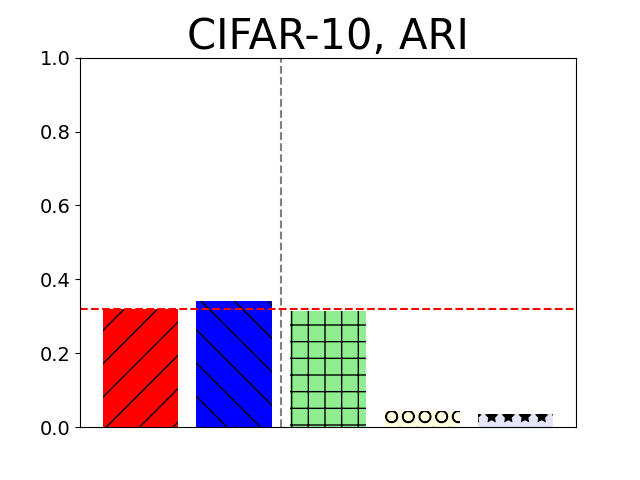}}\hspace{-0.3in}
\hfil
    \subfloat{\includegraphics[width=0.25\linewidth]{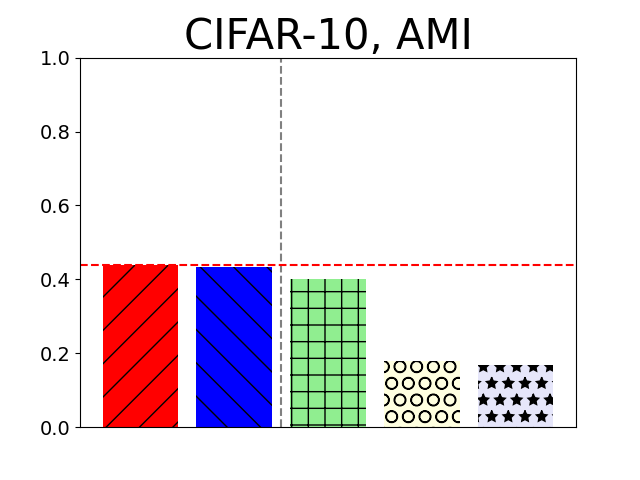}}
\hfil
    \subfloat{\includegraphics[width=0.25\linewidth]{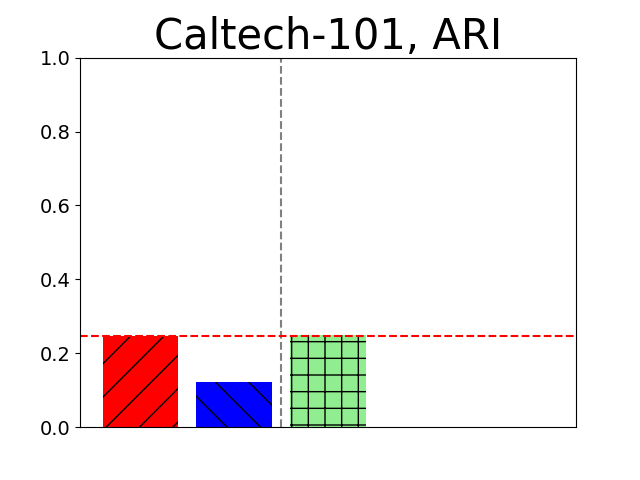}}\hspace{-0.3in}
\hfil
    \subfloat{\includegraphics[width=0.25\linewidth]{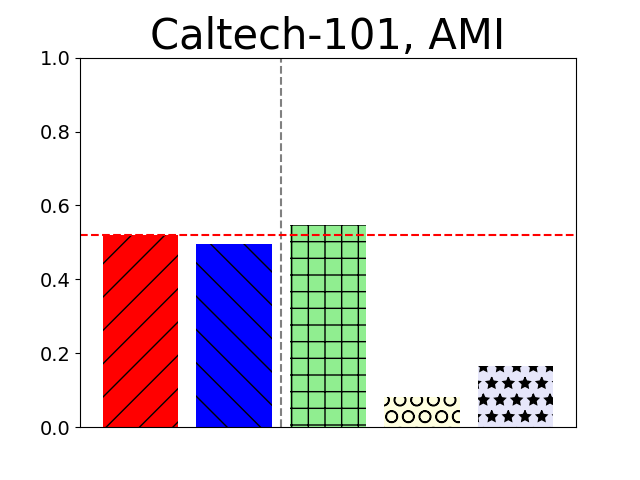}}

    \vspace{-1.2em}
    \subfloat{\includegraphics[width=0.25\linewidth]{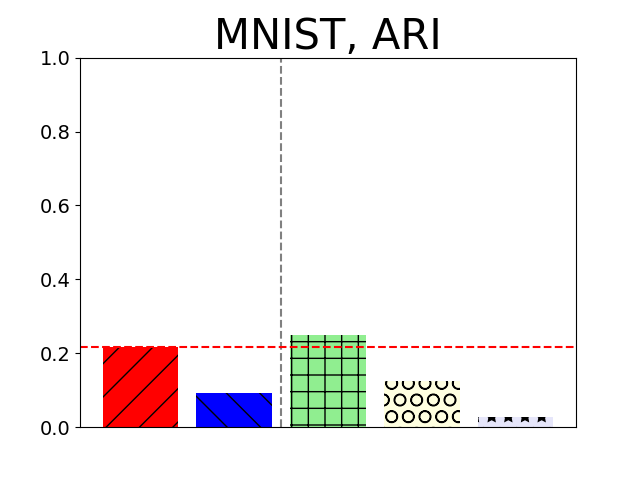}}\hspace{-0.3in}
\hfil
    \subfloat{\includegraphics[width=0.25\linewidth]{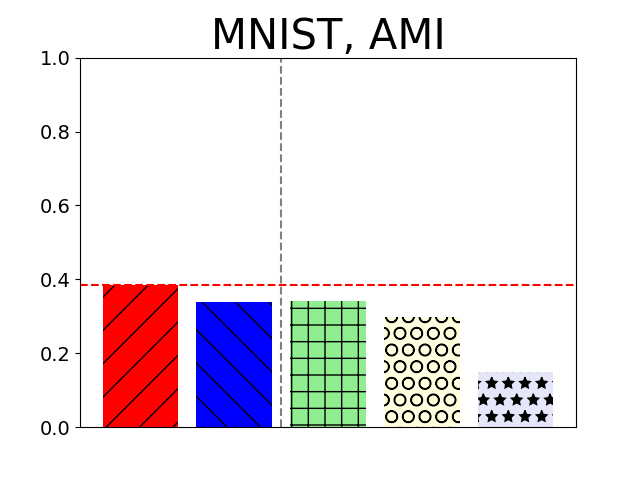}}
\hfil
    \subfloat{\includegraphics[width=0.25\linewidth]{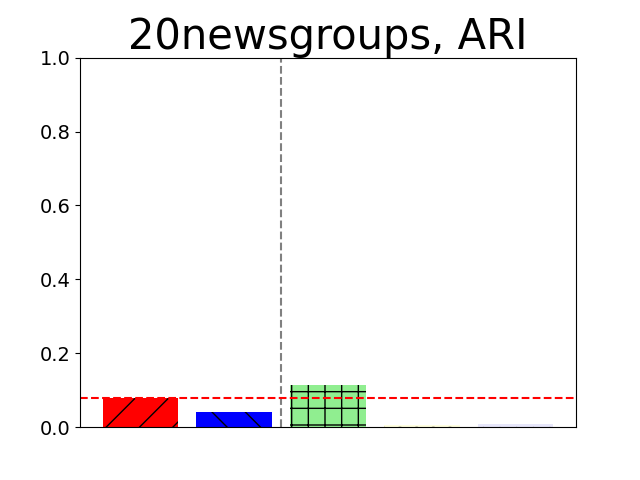}}\hspace{-0.3in}
\hfil
    \subfloat{\includegraphics[width=0.25\linewidth]{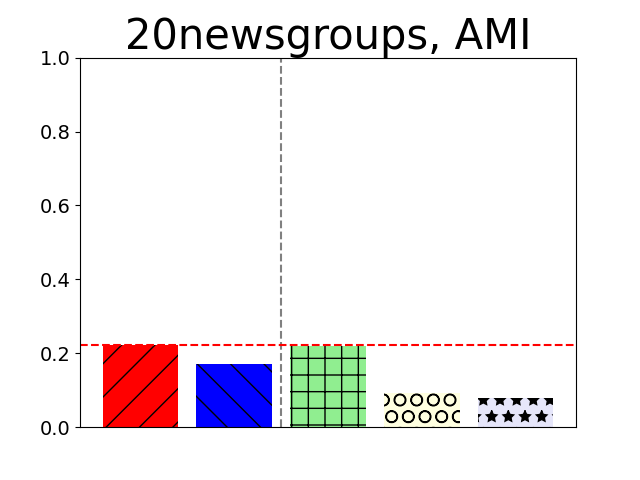}}

    \includegraphics[width=\linewidth]{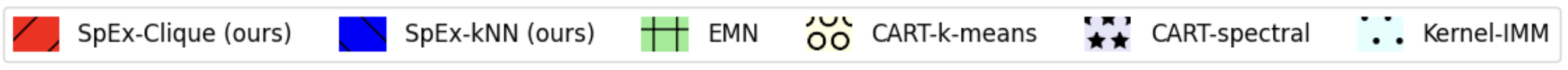}
\caption{Adjusted Rand Index (ARI) and Adjusted Mutual Information (AMI), higher is better.}
\label{fig:res}
\end{center}
\vskip -0.2in
\end{figure*}

We evaluate our methods compared to baselines, on eight public real-world datasets of various sizes and dimensions, detailed in~\Cref{tab:dataset-info}. Our code is available online.\footnote{\url{https://github.com/talargv/SpEx}} 
%
%

\noindent\emph{Our algorithms.}
We focus on two instantiations of our method:
\clique~with spectral clustering as the reference, and \knn where the $k$-NN graph is constructed with $k=20$. 
\Cref{app:experiments} includes additional results for~\clique~with $k$-means and kernel $k$-means as the reference clustering, and for~\knn~with other values of $k$. 

\noindent\emph{Baselines.} We evaluate the following baselines:
\begin{CompactItemize}
    \item EMN~\cite{esfandiari2022almost}, the state of the art for explainable clustering with $k$-means as the reference clustering. Note that EMN cannot work with a spectral or a kernel $k$-means reference clustering, as its operation requires the reference clustering to have centroids.
    \item CART, with both $k$-means and spectral clustering as the reference clustering. 
    \item Kernel IMM~\cite{fleissnerexplaining}, which uses kernel $k$-means as the reference clustering. 
\end{CompactItemize}

Our experiments indicated that Kernel IMM does not scale well, failing to run on the larger scale datasets. We therefore report its results only for the three smaller scale datasets (Breast Cancer, Ecoli, Iris), which are the ones used in~\cite{fleissnerexplaining}. See more in~\Cref{app:experiments}.

\noindent\emph{Evaluation measures.}
The classes of each dataset are used as groundtruth clusters. To evaluate each method, we report two standard (and incomparable) measures of clustering agreement between two clusterings of the same data: Adjusted Rand Index (ARI) \cite{Hubert1985ComparingP} and Adjusted Mutual Information (AMI) \cite{Nguyen2010InformationTM}. Each is explainable clustering is evaluated for its agreement with the groundtruth clustering through these measures.

\subsection{Results}
Main results are reported in~\Cref{fig:res}. 
\Cref{tab:clustering_comparison1,tab:clustering_comparison2} augment them with additional choices of reference clustering for~\clique~and CART and list ARIs with respect to the reference clustering. They also contain two synthetic small-scale datasets, R15 and Pathbased, used in \cite{fleissnerexplaining}.
Appendix \ref{app:experiments} includes additional results.

The results show that \clique~is the most consistently high-performing algorithm. It outperforms each baseline in most cases in both evaluation measures, and is never exceeded by more than one baseline at the same time. 
%
\knn~is highly effective particularly on the low dimensional datasets (Beans, Ecoli and Iris).\footnote{These datasets are of dimension up to 16. The other datasets have dimension at least 30.} On these datasets, it outperforms all baselines as well as~\clique, with the exception of Kernel IMM on Iris.\footnote{On Iris, kernel $k$-means significantly outperforms spectral and $k$-means as the reference clustering. In \Cref{fig:res}, only Kernel IMM uses it a reference. When~\clique~or CART are run on the same kernel $k$-means reference, the results they yield are identical to Kernel IMM.  These results are included in \Cref{app:experiments}.} 

CART performs well on the smaller datasets (confirming similar results reported in~\cite{fleissnerexplaining}), however, it performs poorly on the larger datasets.  
Recall that CART is not originally intended for explainable clustering, but can be repurposed for it since it produces a structurally compatible output (a binary decision tree with coordinate cuts). In~\Cref{sec:introduction,subsec:CART}, we discussed a toy example from~\cite{moshkovitz2020explainable} that demonstrates a failure mode of CART for explainable clustering due to its incompatible objective. Our experimental results indicate this incompatibility may also impede its performance on real data. 



\begin{table*}[t]
\vspace{0.1in}
{\fontsize{9}{11}\selectfont  
\centering
\caption{Results on smaller datasets. Rows are grouped by reference clustering, shown as the first row in each group. The REF column lists the ARI with respect to the reference clustering (rather than the ground truth as in ARI column). Best scores per reference clustering are in bold.}
\label{tab:clustering_comparison1}
}
\vspace{0.1in}
\centering
\resizebox{\textwidth}{!}{
\begin{tabular}{l *{5}{ccc}}
\toprule
& \multicolumn{3}{c}{R15} & \multicolumn{3}{c}{Pathbased} & \multicolumn{3}{c}{Ecoli} & \multicolumn{3}{c}{Iris} & \multicolumn{3}{c}{Cancer}\\
\cmidrule(lr){2-4} \cmidrule(lr){5-7} \cmidrule(lr){8-10} \cmidrule(lr){11-13} \cmidrule(lr){14-16} 
Algorithm & ARI & AMI & REF & ARI & AMI & REF & ARI & AMI & REF & ARI & AMI & REF & ARI & AMI & REF \\
\midrule
\textbf{\emph{REF:}} Spectral
& .993 & .994 & 1. &
.526 & .570 & 1. &
.711 & .653 & 1. &
.630 & .661 & 1. &
.743 & .626 & 1. \\
\knn
& .982 & .987 & .989 &
.332 & .410 & .551 &
\textbf{.679} & \textbf{.642} & .863 &
.450 & \textbf{.647} & .450 &
.507 & .490 & .562 \\
\clique
& \textbf{.986} & \textbf{.989} & \textbf{.993} &
\textbf{.441} & \textbf{.517} & \textbf{.824} &
.662 & .621 & .847 &
\textbf{.576} & \textbf{.629} & \textbf{.787} &
\textbf{.694} & \textbf{.588} & .785 \\
CART
& \textbf{.986} & \textbf{.989} & \textbf{.993} &
\textbf{.441} & \textbf{.517} & \textbf{.824} &
.672 & .618 & \textbf{.886} &
\textbf{.576} & .\textbf{629} & \textbf{.787} &
683 & .560 & \textbf{.811} \\
\midrule
\textbf{\emph{REF:}} $k$-means
& .993 & .994 & 1. &
.461 & .543 & 1. &
.489 & .609 & 1. &
.641 & .669 & 1. &
.491 & .464 & 1. \\
EMN
& \textbf{.986} & \textbf{.989} & \textbf{.993} &
\textbf{.461} & \textbf{.543} & \textbf{1.} &
\textbf{.456} & \textbf{.559} & \textbf{.873} &
\textbf{.576} & \textbf{.629} & \textbf{.772} &
\textbf{.491} & \textbf{.464} & \textbf{1.} \\
\clique
& \textbf{.986} & \textbf{.989} & \textbf{.993} &
\textbf{.461} & \textbf{.543} & \textbf{1.} &
\textbf{.456} & \textbf{.559} & \textbf{.873} &
\textbf{.576} & \textbf{.629} & \textbf{.772} &
\textbf{.491} & \textbf{.464} & \textbf{1.} \\
CART
& \textbf{.986} & \textbf{.989} & \textbf{.993} &
.421 & .507 & .897 &
.454 & .543 & .840 &
\textbf{.576} & \textbf{.629} & \textbf{.772} &
\textbf{.491} & \textbf{.464} & \textbf{1.} \\
\midrule
\textbf{\emph{REF:}} Kernel $k$-means
& .908 & .967 & 1. &
.919 & .888 & 1. &
.538 & .612 & 1. &
.731 & .767 & 1. &
.116 & .228 & 1. \\
Kernel IMM
& \textbf{.904} & \textbf{.962} & \textbf{.986} &
\textbf{.614} & \textbf{.614} & \textbf{.583} &
.522 & .560 & .848 &
\textbf{.732} & \textbf{.788} & \textbf{.924} &
.127 & .241 & \textbf{.93} \\
\clique
& .869 & .941 & .951 &
.479 & .553 & .450 &
\textbf{.529} & \textbf{.573} & \textbf{.851} &
\textbf{.732} & \textbf{.788} & \textbf{.924} &
\textbf{.406} & \textbf{.414} & .511 \\
CART
& .682 & .876 & .759 &
.479 & .553 & .450 &
.500 & .558 & .824 &
\textbf{.732} & \textbf{.788} & \textbf{.924} &
\textbf{.406} & \textbf{.414} & .511 \\
\bottomrule
\end{tabular}
}
\end{table*}


\begin{table*}[t]
\vspace{0.1in}
{\fontsize{9}{11}\selectfont  
\centering
\caption{Results on larger datasets.}
\label{tab:clustering_comparison2}
}
\vspace{0.1in}
\centering
\resizebox{\textwidth}{!}{
\begin{tabular}{l *{5}{ccc}}
\toprule
& \multicolumn{3}{c}{MNIST} & \multicolumn{3}{c}{Caltech 101} & \multicolumn{3}{c}{Newsgroups} & \multicolumn{3}{c}{Beans} & \multicolumn{3}{c}{Cifar}\\
\cmidrule(lr){2-4} \cmidrule(lr){5-7} \cmidrule(lr){8-10} \cmidrule(lr){11-13} \cmidrule(lr){14-16}
Algorithm & ARI & AMI & REF & ARI & AMI & REF & ARI & AMI & REF & ARI & AMI & REF & ARI & AMI & REF \\
\midrule
\textbf{\emph{REF:}} Spectral
& .745 & .820 & 1. &
.563 & .859 & 1. &
.431 & .671 & 1. &
.586 & .677 & 1. &
.712 & .801 & 1. \\
\knn
& .092 & .338 & .150 &
.121 & .497 & .168 &
.042 & .170 & .09 &
\textbf{.574} & \textbf{.690} & .649 &
\textbf{.342} & .434 & .373 \\
\clique
& \textbf{.217} & \textbf{.384} & \textbf{.282} &
\textbf{.247} & \textbf{.521} & \textbf{.303} &
\textbf{.078} & \textbf{.223} & \textbf{.189} &
.564 & .671 & \textbf{.743} &
.320 & \textbf{.438} & \textbf{.394} \\
CART
& .027 & .148 & .030 &
-.010 & .166 & .005 &
.008 & .078 & -.013 &
.542 & .658 & .705 &
.036 & .169 & .035 \\
\midrule
\textbf{\emph{REF:}} $k$-means
& .364 & .481 & 1. &
.405 & .822 & 1. &
.502 & .660 & 1. &
.572 & .689 & 1. &
.636 & .738 & 1. \\
EMN
& \textbf{.25} & \textbf{.342} & \textbf{.42} &
\textbf{.249} & \textbf{.548} & \textbf{.416} &
\textbf{.115} & .219 & \textbf{.163} &
\textbf{.563} & \textbf{.688} & \textbf{.780} &
\textbf{.314} & .402 & \textbf{.387} \\
\clique
& .209 & .336 & .403 &
.122 & .495 & .195 &
.098 & \textbf{.249} & .128 &
.562 & .687 & .773 &
.288 & \textbf{.410} & .370 \\
CART
& .124 & .299 & .229 &
-.017 & .082 & .003 &
.005 & .092 & .006 &
.536 & .669 & .757 &
.045 & .180 & .037 \\
\bottomrule
\end{tabular}
}
\end{table*}



\paragraph{Conclusion and limitations.}
Prior work on explainable clustering has focused on a few specific reference objectives, allowing it to prove worst-case approximation bounds tailored to them (cf.~\Cref{sec:related}). \spex, on the other hand, fills the gap of a generic, reference-oblivious method, currently lacking in the literature. Thus, while theoretically well-grounded (\Cref{thm:geomcheeger}, \Cref{sec:nonuniform}), it cannot offer such bounds. 
Other methods may be better choices for some specific reference objectives -- e.g., EMN is known to be near-optimal for $k$-means, and \Cref{tab:clustering_comparison1,tab:clustering_comparison2} indeed show that it is generally better than \spex\ for that reference objective. \spex\ is better for data where objectives already well-studied in explainable clustering (like $k$-means) fall short, and more versatile objectives (like spectral clustering) are needed. 
As future work, it would be interesting to explore if the theoretical framework we develop can yet lead to formal approximation bounds for some of these objectives. 

\section*{Acknowledgements}
This work was supported by Len Blavatnik and the Blavatnik Family foundation and by an Alon Scholarship of the Israeli Council for Higher Education. TW is also with Amazon; this work is not associated with Amazon.





\bibliography{explainableclustering}

\begin{thebibliography}{10}

\bibitem{alon1985lambda1}
Noga Alon and Vitali~D Milman.
\newblock $\lambda$1, isoperimetric inequalities for graphs, and superconcentrators.
\newblock {\em Journal of Combinatorial Theory, Series B}, 38(1):73--88, 1985.

\bibitem{andoni2018data}
Alexandr Andoni, Assaf Naor, Aleksandar Nikolov, Ilya Razenshteyn, and Erik Waingarten.
\newblock Data-dependent hashing via nonlinear spectral gaps.
\newblock In {\em Proceedings of the 50th annual ACM SIGACT symposium on theory of computing}, pages 787--800, 2018.

\bibitem{andoni2018holder}
Alexandr Andoni, Assaf Naor, Aleksandar Nikolov, Ilya Razenshteyn, and Erik Waingarten.
\newblock H{\"o}lder homeomorphisms and approximate nearest neighbors.
\newblock In {\em 2018 IEEE 59th Annual Symposium on Foundations of Computer Science (FOCS)}, pages 159--169. IEEE, 2018.

\bibitem{bandyapadhyay2023find}
Sayan Bandyapadhyay, Fedor~V Fomin, Petr~A Golovach, William Lochet, Nidhi Purohit, and Kirill Simonov.
\newblock How to find a good explanation for clustering?
\newblock {\em Artificial Intelligence}, 322:103948, 2023.

\bibitem{bay2000uci}
Stephen~D Bay, Dennis Kibler, Michael~J Pazzani, and Padhraic Smyth.
\newblock The uci kdd archive of large data sets for data mining research and experimentation.
\newblock {\em ACM SIGKDD explorations newsletter}, 2(2):81--85, 2000.

\bibitem{breiman1984classification}
L~Breiman, JH~Friedman, R~Olshen, and CJ~Stone.
\newblock Classification and regression trees.
\newblock 1984.

\bibitem{charikar2022near}
Moses Charikar and Lunjia Hu.
\newblock Near-optimal explainable k-means for all dimensions.
\newblock In {\em Proceedings of the 2022 Annual ACM-SIAM Symposium on Discrete Algorithms (SODA)}, pages 2580--2606. SIAM, 2022.

\bibitem{chung1997spectral}
Fan~RK Chung.
\newblock {\em Spectral graph theory}, volume~92.
\newblock American Mathematical Soc., 1997.

\bibitem{cohen2023interpretable}
Eldan Cohen.
\newblock Interpretable clustering via soft clustering trees.
\newblock In {\em International Conference on Integration of Constraint Programming, Artificial Intelligence, and Operations Research}, pages 281--298. Springer, 2023.

\bibitem{dhillon2004kernel}
Inderjit~S Dhillon, Yuqiang Guan, and Brian Kulis.
\newblock Kernel k-means: spectral clustering and normalized cuts.
\newblock In {\em Proceedings of the tenth ACM SIGKDD international conference on Knowledge discovery and data mining}, pages 551--556, 2004.

\bibitem{donglearning}
Yihe Dong, Piotr Indyk, Ilya Razenshteyn, and Tal Wagner.
\newblock Learning space partitions for nearest neighbor search.
\newblock In {\em International Conference on Learning Representations}, 2020.

\bibitem{esfandiari2022almost}
Hossein Esfandiari, Vahab Mirrokni, and Shyam Narayanan.
\newblock Almost tight approximation algorithms for explainable clustering.
\newblock In {\em Proceedings of the 2022 Annual ACM-SIAM Symposium on Discrete Algorithms (SODA)}, pages 2641--2663. SIAM, 2022.

\bibitem{FeiFei2006OneshotLO}
Li~Fei-Fei, Rob Fergus, and Pietro Perona.
\newblock One-shot learning of object categories.
\newblock {\em IEEE Transactions on Pattern Analysis and Machine Intelligence}, 28:594--611, 2006.

\bibitem{Fisher1936THEUO}
Rory~A. Fisher.
\newblock The use of multiple measurements in taxonomic problems.
\newblock {\em Annals of Human Genetics}, 7:179--188, 1936.

\bibitem{fleissnerexplaining}
Maximilian Fleissner, Leena~Chennuru Vankadara, and Debarghya Ghoshdastidar.
\newblock Explaining kernel clustering via decision trees.
\newblock In {\em The Twelfth International Conference on Learning Representations (ICLR)}, 2024.

\bibitem{fraiman2013interpretable}
Ricardo Fraiman, Badih Ghattas, and Marcela Svarc.
\newblock Interpretable clustering using unsupervised binary trees.
\newblock {\em Advances in Data Analysis and Classification}, 7(2):125--145, 2013.

\bibitem{frost2020exkmc}
Nave Frost, Michal Moshkovitz, and Cyrus Rashtchian.
\newblock Exkmc: Expanding explainable $ k $-means clustering.
\newblock {\em arXiv preprint arXiv:2006.02399}, 2020.

\bibitem{gabidolla2022optimal}
Magzhan Gabidolla and Miguel~{\'A} Carreira-Perpi{\~n}{\'a}n.
\newblock Optimal interpretable clustering using oblique decision trees.
\newblock In {\em Proceedings of the 28th ACM SIGKDD Conference on Knowledge Discovery and Data Mining}, pages 400--410, 2022.

\bibitem{gamlath2021nearly}
Buddhima Gamlath, Xinrui Jia, Adam Polak, and Ola Svensson.
\newblock Nearly-tight and oblivious algorithms for explainable clustering.
\newblock {\em Advances in Neural Information Processing Systems}, 34:28929--28939, 2021.

\bibitem{gupta2023price}
Anupam Gupta, Madhusudhan~Reddy Pittu, Ola Svensson, and Rachel Yuan.
\newblock The price of explainability for clustering.
\newblock In {\em 2023 IEEE 64th Annual Symposium on Foundations of Computer Science (FOCS)}, pages 1131--1148. IEEE, 2023.

\bibitem{Horton1996APC}
Paul Horton and Kenta Nakai.
\newblock A probabilistic classification system for predicting the cellular localization sites of proteins.
\newblock {\em Proceedings. International Conference on Intelligent Systems for Molecular Biology}, 4:109--15, 1996.

\bibitem{Hubert1985ComparingP}
Lawrence~J. Hubert and Phipps Arabie.
\newblock Comparing partitions.
\newblock {\em Journal of Classification}, 2:193--218, 1985.

\bibitem{Koklu2020MulticlassCO}
Murat Koklu and Ilker~Ali {\"O}zkan.
\newblock Multiclass classification of dry beans using computer vision and machine learning techniques.
\newblock {\em Comput. Electron. Agric.}, 174:105507, 2020.

\bibitem{Krizhevsky2009LearningML}
Alex Krizhevsky.
\newblock Learning multiple layers of features from tiny images.
\newblock 2009.

\bibitem{laber2023shallow}
Eduardo Laber, Lucas Murtinho, and Felipe Oliveira.
\newblock Shallow decision trees for explainable k-means clustering.
\newblock {\em Pattern Recognition}, 137:109239, 2023.

\bibitem{laber2021price}
Eduardo~S Laber and Lucas Murtinho.
\newblock On the price of explainability for some clustering problems.
\newblock In {\em International Conference on Machine Learning}, pages 5915--5925. PMLR, 2021.

\bibitem{laber2024computational}
Eduardo~Sany Laber.
\newblock The computational complexity of some explainable clustering problems.
\newblock {\em Information Processing Letters}, 184:106437, 2024.

\bibitem{laber2023nearly}
Eduardo~Sany Laber and Lucas~Saadi Murtinho.
\newblock Nearly tight bounds on the price of explainability for the k-center and the maximum-spacing clustering problems.
\newblock {\em Theoretical Computer Science}, 949:113744, 2023.

\bibitem{LeCun1998GradientbasedLA}
Yann LeCun, L{\'e}on Bottou, Yoshua Bengio, and Patrick Haffner.
\newblock Gradient-based learning applied to document recognition.
\newblock {\em Proc. IEEE}, 86:2278--2324, 1998.

\bibitem{makarychev2021near}
Konstantin Makarychev and Liren Shan.
\newblock Near-optimal algorithms for explainable k-medians and k-means.
\newblock In {\em International Conference on Machine Learning}, pages 7358--7367. PMLR, 2021.

\bibitem{makarychev2022explainable}
Konstantin Makarychev and Liren Shan.
\newblock Explainable k-means: don’t be greedy, plant bigger trees!
\newblock In {\em Proceedings of the 54th Annual ACM SIGACT Symposium on Theory of Computing}, pages 1629--1642, 2022.

\bibitem{makarychev2024random}
Konstantin Makarychev and Liren Shan.
\newblock Random cuts are optimal for explainable k-medians.
\newblock {\em Advances in Neural Information Processing Systems}, 36, 2024.

\bibitem{moshkovitz2020explainable}
Michal Moshkovitz, Sanjoy Dasgupta, Cyrus Rashtchian, and Nave Frost.
\newblock Explainable k-means and k-medians clustering.
\newblock In {\em International conference on machine learning}, pages 7055--7065. PMLR, 2020.

\bibitem{Nguyen2010InformationTM}
Xuan~Vinh Nguyen, Julien Epps, and James Bailey.
\newblock Information theoretic measures for clusterings comparison: Variants, properties, normalization and correction for chance.
\newblock {\em J. Mach. Learn. Res.}, 11:2837--2854, 2010.

\bibitem{ohl2024kernel}
Louis Ohl, Pierre-Alexandre Mattei, Micka{\"e}l Leclercq, Arnaud Droit, and Fr{\'e}d{\'e}ric Precioso.
\newblock Kernel kmeans clustering splits for end-to-end unsupervised decision trees.
\newblock {\em arXiv preprint arXiv:2402.12232}, 2024.

\bibitem{quinlan1986induction}
J.~Ross Quinlan.
\newblock Induction of decision trees.
\newblock {\em Machine learning}, 1:81--106, 1986.

\bibitem{Radford2021LearningTV}
Alec Radford, Jong~Wook Kim, Chris Hallacy, Aditya Ramesh, Gabriel Goh, Sandhini Agarwal, Girish Sastry, Amanda Askell, Pamela Mishkin, Jack Clark, Gretchen Krueger, and Ilya Sutskever.
\newblock Learning transferable visual models from natural language supervision.
\newblock In {\em International Conference on Machine Learning}, 2021.

\bibitem{reimers-2019-sentence-bert}
Nils Reimers and Iryna Gurevych.
\newblock Sentence-bert: Sentence embeddings using siamese bert-networks.
\newblock In {\em Proceedings of the 2019 Conference on Empirical Methods in Natural Language Processing}. Association for Computational Linguistics, 11 2019.

\bibitem{scholkopf1996nonlinear}
B~Sch{\"o}lkopf, AJ~Smola, and KR~M{\"u}ller.
\newblock Nonlinear component analysis as a kernel eigenvalue problem.
\newblock 1996.

\bibitem{shati2023optimal}
Pouya Shati, Eldan Cohen, and Sheila McIlraith.
\newblock Optimal decision trees for interpretable clustering with constraints (extended version).
\newblock {\em arXiv preprint arXiv:2301.12671}, 2023.

\bibitem{shi2000normalized}
Jianbo Shi and Jitendra Malik.
\newblock Normalized cuts and image segmentation.
\newblock {\em IEEE Transactions on pattern analysis and machine intelligence}, 22(8):888--905, 2000.

\bibitem{spielman2025spectral}
Daniel~A. Spielman.
\newblock {\em Spectral and Algebraic Graph Theory}.
\newblock Textbook draft available online, 2025.

\bibitem{Street1993NuclearFE}
William~Nick Street, William~H. Wolberg, and Olvi~L. Mangasarian.
\newblock Nuclear feature extraction for breast tumor diagnosis.
\newblock In {\em Electronic imaging}, 1993.

\bibitem{trevisan2013cheeger}
Luca Trevisan.
\newblock Is cheeger-type approximation possible for nonuniform sparsest cut?
\newblock {\em arXiv preprint arXiv:1303.2730}, 2013.

\bibitem{upadhya2024neurcam}
Nakul Upadhya and Eldan Cohen.
\newblock Neurcam: Interpretable neural clustering via additive models.
\newblock {\em arXiv preprint arXiv:2408.13361}, 2024.

\bibitem{von2007tutorial}
Ulrike Von~Luxburg.
\newblock A tutorial on spectral clustering.
\newblock {\em Statistics and computing}, 17:395--416, 2007.

\bibitem{wei1989towards}
Yen-Chuen Wei and Chung-Kuan Cheng.
\newblock Towards efficient hierarchical designs by ratio cut partitioning.
\newblock In {\em 1989 IEEE International Conference on Computer-Aided Design. Digest of Technical Papers}, pages 298--301. IEEE, 1989.

\end{thebibliography}
\bibliographystyle{plain}



\newpage
\appendix
\onecolumn

\section*{Impact Statement}
This paper is concerned with fitting interpretable explanations to decisions made by machine learning algorithms, in this case, grouping of individual points into clusters. While explainability in AI is intended to promote safe and ethical use in machine learning, any deployment of automated systems in settings consequential to human individuals may have unforeseen consequences. While our work is not associated with any particular outstanding or new risks, every real-world use of an explainability mechanism in AI should be used with caution and incorporate human oversight, testing, and other use case specific safeguards to mitigate any associated risks.

\section{Proofs}\label{sec:proofs}

In \Cref{sec:mainthmproof}, we provide a proof of \Cref{thm:geomcheeger}. 
In \Cref{sec:corollaries}, as a concrete demonstration, we show an application of it that directly relates the $k$-means reference clustering cost to the graph conductance in a certain graph that describes the reference clustering. 
In \Cref{sec:immproof}, we provide the modified IMM analysis discussed in \Cref{sec:imm}.

\subsection{Proof of~\Cref{thm:geomcheeger}}\label{sec:mainthmproof}

We restate the theorem:
\begin{theorem}[\Cref{thm:geomcheeger}, restated]
    Let $X\subset\R^d$ be a set of points, where $x\in X$ has coordinates $x=(x_1,\ldots,x_d)$. Let $G(X,E,w)$ be a graph with vertex set $X$. Consider the following two distributions over pairs in of points in $X$:
    \begin{CompactItemize}
        \item $\mathcal D_{\mathrm{adj}}$ is the distribution over adjacent pairs in $G$, where a pair $x,y\in E$ is sampled with probability proportional to the edge weight between them.
        \item $\mathcal D_{\mathrm{all}}$ is the distribution over all pairs $x,y\in X$, where $x$ and $y$ are sampled independently, each with probability proportional to its weighted degree in $G$. 
    \end{CompactItemize} 
    Then, there is a valid coordinate cut $j,\tau$
    such that
    \[ \Psi_G(S_{j,\tau}(X)) \leq \sqrt{\frac{\E_{x,y\sim\mathcal D_{\mathrm{adj}}}\norm{x-y}_2^2}{\E_{x,y\sim\mathcal D_{\mathrm{all}}}\norm{x-y}_2^2}} . \]
\end{theorem}

The proof uses arguments that are standard in spectral graph theory, and follows \cite{donglearning} in a generalized form (with an arbitrary graph). We include it here for completeness and clarity.  

Recall the following standard useful fact,
\begin{fact}\label{clm:dansfav}
    Let $a_1,\ldots,a_n\geq0$ and $b_1,\ldots,b_n\geq0$ such that $\sum_{i}b_i>0$. Then
    \[ \min_i\frac{a_i}{b_i} \leq \frac{\sum_ia_i}{\sum_i b_i} \leq \max_i\frac{a_i}{b_i}.  \]
\end{fact}
Using it we have,
\begin{align*}
\frac{\E_{x,y\sim\mathcal D_{\mathrm{adj}}}\norm{x-y}_2^2}{\E_{x,y\sim\mathcal D_{\mathrm{all}}}\norm{x-y}_2^2} &= \frac{\E_{x,y\sim\mathcal D_{\mathrm{adj}}}\sum_{i=1}^d|x_i-y_i|^2}{\E_{x,y\sim\mathcal D_{\mathrm{all}}}\sum_{i=1}^d|x_i-y_i|^2} \\
&= \frac{\sum_{i=1}^d\E_{x,y\sim\mathcal D_{\mathrm{adj}}}|x_i-y_i|^2}{\sum_{i=1}^d\E_{x,y\sim\mathcal D_{\mathrm{all}}}|x_i-y_i|^2} \\
&\geq \min_{i^*\in\{1,\ldots,d\}}\frac{\E_{x,y\sim\mathcal D_{\mathrm{adj}}}|x_{i^*}-y_{i^*}|^2}{\E_{x,y\sim\mathcal D_{\mathrm{all}}}|x_{i^*}-y_{i^*}|^2} . \numberthis\label{eq:istar}
\end{align*}  

Let $n=|X|$ be the number of nodes in $G$. For convenience we arbitrarily label them $1,\ldots,n$ so we can index the entries a vector $z\in\R^n$ by points in $X$, so the coordinates of $z$ are $(z(x))_{x\in X}$. We do so similarly for matrices in $\R^{n\times n}$. 

Recall that $w(x,y)$ is the edge weight between $x$ and $y$ in $G$. Let $d_G(x)$ denote the weighted degree of $x$ in $G$. 
Let $\Delta_G=\sum_{x,y} w(x,y)$ be the sum of all edge weights. Let $A_G\in\R^{n\times n}$ be its weighted adjacency matrix, $A_G(x,y)=w(x,y)$. Let $D_G\in\R^{n\times n}$ be its diagonal matrix of degrees, $D_G(x,x)=d_G(x)$. Let $L_G=D_G-A_G$ be its Laplacian matrix. 
A standard fact in spectral graph theory is the identity $z^TL_Gz=\sum_{x,y}w(x,y)|z(x)-z(y)|^2$ for every $z\in\R^n$. 

Let $z_*\in\R^n$ be the vector with entries defined by $z_*(x)=x_{i^*}$, where $i^*\in\{1,\ldots,d\}$ is the minimizer from \Cref{eq:istar}. 
Recalling that $\mathcal{D_{\mathrm{adj}}}$ is the distribution over pairs $x,y$ with probability mass $\frac{w(x,y)}{\Delta_G}$, we have
\begin{align*}
    \E_{x,y\sim\mathcal D_{\mathrm{adj}}}|x_{i^*}-y_{i^*}|^2 &= \sum_{x,y}\frac{w(x,y)}{\Delta_G}|x_{i^*}-y_{i^*}|^2 \\
    &= \sum_{x,y}\frac{w(x,y)}{\Delta_G}|z_*(x)-z_*(y)|^2 \\
    &= \frac{z_*^TL_Gz_*}{\Delta_G} . \numberthis\label{eq:dadj}
\end{align*}

Let $H$ be a weighted clique over $X$ in which the edge weight between every pair $x,y\in X$ is
\begin{equation}\label{eq:hweights}
w_H(x,y)=\frac{d_G(x)\cdot d_G(y)}{2\Delta_G} .
\end{equation}
This is the $G$-degree weighted clique over $X$ mentioned in \Cref{sec:nonuniform}, except we scale all weights down by $2\Delta_G$. The weighted degree of $x$ in $H$ is
\begin{align*}
  d_H(x) &= \sum_{y\neq x}w_H(x,y) & \\
  &= \frac{d_G(x)}{\Delta_G}\sum_{y\neq x}d_G(y) & \\
  &= \frac{d_G(x)}{2\Delta_G}(2\Delta_G-d_G(x)) & \text{since $\Delta_G = \sum_{x,y}w(x,y) = \frac12\sum_{y\in X}d_G(y)$} \\
  &= d_G(x) - \frac{(d_G(x))^2}{2\Delta_G}  . & \numberthis\label{eq:dhdg}
\end{align*}
Let $A_H$, $D_H$ and $L_H=D_H-A_H$ be the weighted adjacency matrix of $H$, its diagonal degree matrix, and its Laplacian matrix, respectively. 
Let $v_G\in\R^n$ be the vector of weighted degrees in $G$, scaled down by $\sqrt{2\Delta_G}$:
\[v_G(x)=\frac{d_G(x)}{\sqrt{2\Delta_G}} . \]
Observe that, by~\Cref{eq:hweights,eq:dhdg},
\[
D_H(x,y) = \begin{cases}
    d_G(x) - \frac{(d_G(x))^2}{2\Delta_G} & \text{if $x=y$} \\
    0 & \text{if $x\neq y$}
    \end{cases}
    \quad \text{and} \quad 
 A_H(x,y) = \begin{cases}
    0 & \text{if $x=y$} \\
    \frac{d_G(x)\cdot d_H(y)}{2\Delta_G} & \text{if $x\neq y$}  
    \end{cases}
\]
Therefore,
\begin{equation}\label{eq:laplacian_h}
L_H=D_H-A_H = D_G-v_Gv_G^T.
\end{equation}

Now, recall that $\mathcal{D_{\mathrm{all}}}$ is the distribution over pairs $x,y$, where $x$ and $y$ are i.i.d., each with probability proportional to its degree in $G$.  
Hence, 
\begin{align*}
    \E_{x,y\sim\mathcal D_{\mathrm{all}}}|x_{i^*}-y_{i^*}|^2 &= \sum_{x,y}\frac{d_G(x)}{2\Delta_G}\cdot\frac{d_G(y)}{2\Delta_G}\cdot|x_{i^*}-y_{i^*}|^2  & \\
    &= \frac{1}{2\Delta_G}\sum_{x,y}w_H(x,y)|z_*(x)-z_*(y)|^2 & \text{by~\Cref{eq:hweights}} \\
    &= \frac{z_*^TL_Hz_*}{2\Delta_G}
    . &\numberthis\label{eq:dall}
\end{align*}
Plugging \Cref{eq:dadj,eq:dall} into \Cref{eq:istar}, we have obtained,
\begin{equation}\label{eq:preprecheeger}
\frac{\E_{x,y\sim\mathcal D_{\mathrm{adj}}}\norm{x-y}_2^2}{\E_{x,y\sim\mathcal D_{\mathrm{all}}}\norm{x-y}_2^2} \geq \frac{2\cdot z_*^TL_Gz_*}{z_*^TL_Hz_*} .
\end{equation}

Finally, let $\mathbf0$ and $\mathbf1$ denote the all-0 and all-1 vectors in $\R^n$.  
Let $\gamma_G$ be the constant $\gamma_G=v_G^Tz_*/v_G^T\mathbf1$. Letting
\[ \tilde z=z_*-\gamma_G\mathbf1 , \]
we have
\begin{equation}\label{eq:perp}
  v_G^T\tilde z = 
  v_G^T(z_* - \gamma_G\mathbf1) = 
  v_G^Tz_* - \frac{v_G^Tz_*}{v_G^T\mathbf1}\cdot v_G^T\mathbf1 = 0.
\end{equation}
It is a standard fact that every graph Laplacian $L$ satisfies $L\mathbf1=\mathbf0$. Therefore,
\[
  L_Gz_* = L_G(\tilde z + \gamma_G\mathbf1) = L_G\tilde z ,  
\]
and, using \Cref{eq:laplacian_h,eq:perp},
\[
  L_Hz_* = L_H(\tilde z + \gamma_G\mathbf1) = L_H\tilde z = (D_G-v_Gv_G^T)\tilde z = D_G\tilde z.
\]
Plugging these into \Cref{eq:preprecheeger}, we obtain
\begin{equation}\label{eq:precheeger}
\frac{\E_{x,y\sim\mathcal D_{\mathrm{adj}}}\norm{x-y}_2^2}{\E_{x,y\sim\mathcal D_{\mathrm{all}}}\norm{x-y}_2^2} \geq 2\cdot\frac{\tilde z^TL_G\tilde z}{\tilde z^TD_G\tilde z} .
\end{equation}

Now we can state the Cheeger inequality for graphs, or rather, a somewhat more general and useful form of it (see, e.g., \cite[Theorem 21.1.3]{spielman2025spectral}):
\begin{theorem}\label{thm:cheeger}
    Let $G(X,E,w)$ be an undirected weighted graph with $n$ nodes. Let $A_G$ be its weighted adjacency matrix, $v_G$ the vector weighted degrees, $D_G$ the diagonal matrix of weighted degrees, and $L_G=D_G-A_G$ the Laplacian. 

    Suppose we have a vector $\tilde z\in\R^n$ that satisfies $\tilde z^Tv_G=0$. Then, there is a threshold $\tilde\tau\in\R$ such that if we consider the cut $S=\{x\in X:\tilde z(x)\leq\tilde\tau\}$, it satisfies $S\neq\emptyset$, $S\neq X$, and 
    \begin{equation}\label{eq:cheeger}
      \Psi_G(S) \leq \sqrt{2\cdot\frac{\tilde z^TL_G\tilde z}{\tilde z^TD_G\tilde z}} .
    \end{equation}
\end{theorem}
\Cref{thm:cheeger} together with \Cref{eq:precheeger} imply a cut with conductance
\[
  \Psi_G(S) \leq \sqrt{\frac{\E_{x,y\sim\mathcal D_{\mathrm{adj}}}\norm{x-y}_2^2}{\E_{x,y\sim\mathcal D_{\mathrm{all}}}\norm{x-y}_2^2}} ,
\]
given by thresholding the coordinates of $\tilde z$ at $\tilde\tau$. Recalling that for every $x\in X$ it holds that
\[ \tilde z(x) = z_*(x) - \gamma_G =x_{i^*} - \gamma_G, \]
the cut $S$ can be equivalently described as $S=\{x\in X : x_{i^*} \leq \tilde\tau + \gamma_G\}$. Therefore, it is a coordinate cut with coordinate $i^*$ and threshold $\tau=\tilde\tau + \gamma_G$, and \Cref{thm:geomcheeger} is proven. \qed

\begin{remark}
For context, we relate \Cref{thm:cheeger} to the more familiar form of Cheeger's inequality. Let $\hat L_G=D^{-1/2}L_GD^{-1/2}$ be the normalized Laplacian matrix of $G$. It is positive semi-definite and its smallest eigenvalue is $0$. Let $\lambda_G$ be its second smallest eigenvalue, and $z_G$ a corresponding eigenvector. It can be shown that the quotient at right-hand side of \Cref{eq:cheeger} is minimized by choosing $\tilde z=D^{1/2}z_G$, and that the minimal value it takes is $\lambda_G$. Hence the more standard form of Cheeger's inequality, $\Psi_G(S)\leq\sqrt{2\lambda_G}$. 
\end{remark}

\subsection{$k$-Means Example of~\Cref{thm:geomcheeger}}\label{sec:corollaries}

As a demonstrative example, let us show a special cases of~\Cref{thm:geomcheeger}, in which a reference $k$-means clustering is described by a suitable graph, and observe the conductance bounds it yields. Let $\mathcal C=\{C_1,\ldots,C_k\}$ be a partition of $X$ into $k$ clusters with corresponding centroids $\mu^{(1)},\ldots,\mu^{(k)}$. In $k$-means, the centroids are given by cluster means, $\mu^{(i)}=\tfrac{1}{|C_i|}\sum_{x\in C_i}x$. Denote the cost of cluster $C_i$ by $\mathrm{cost}(C_i)=\sum_{x\in C_i}\norm{x-\mu^{(i)}}_2^2$, and the total reference clustering cost by 
\[ \mathrm{cost}(\mathcal C)=\sum_{i=1}^k\mathrm{cost}(C_i)=\sum_{i=1}^k\sum_{x\in C_i}\norm{x-\mu^{(i)}}_2^2 , \]



We instantiate~\Cref{thm:geomcheeger} with a clique graph with edge weights inversely proportional to their cluster sizes.
\begin{corollary}\label{prp:cliquegraph}
    In the weighted clique graph where every edge connecting $x,y\in C_i$ has weight $1/(|C_i|-1)$, there is a coordinate cut with conductance at most
    \[ \sqrt{2\cdot\frac{\mathrm{cost}(\mathcal C)}{\tfrac1{|X|}\sum_{x,y\in X}\norm{x-y}_2^2}} . \]
\end{corollary}
The proof uses the following fact which can be verified by direct substitution.
\begin{fact} \label{clm_norm_squared}
    For every set of points $C$ with mean $\mu=\frac{1}{|C|}\sum_{x\in C}x$,
    \[ \sum_{x\in C}\norm{x-\mu}_2^2= \frac{1}{|C|}\sum_{x\neq y\in C}\norm{x-y}_2^2 . \]
\end{fact}
\begin{proof}
Denote the weighted clique graph by $G(X, X^2, w)$.
For this graph:
\[w(x,y)=\begin{cases}
    \frac{1}{|C_i|-1}& \exists \; i \;x,y\in C_i \\
    0 & else
\end{cases}\]
(Note that if $|C_i|=1$ then the only node in $C_i$ has no incident edges, thus the above setting of edge weights is well-defined.) 
Thus, it holds that:
\[\forall i,x\in C_i \; \;; \;\; d_G(x)=\sum_{\substack{y\in C_i \\ y\neq x}} \frac{1}{|C_i|-1} = 1 .\]
Hence, \Cref{thm:geomcheeger} implies the existence of a coordinate cut for which the conductance is at most
\begin{align*}
    \sqrt{\frac{\mathbb{E}_{x,y\sim \mathcal{D}_{adj}}\Vert x-y\Vert_2^2}{\mathbb{E}_{x,y\sim \mathcal{D}_{all}}\Vert x-y\Vert_2^2}} 
    &=\sqrt{\frac{\sum_{x,y\in X} \frac{w(x,y)}{\sum_{x',y'\in X}w(x',y')}\Vert x-y\Vert_2^2}{\sum_{x,y\in X}\frac{d_G(x)d_G(y)}{\sum_{x',y'\in X}d_G(x)d_G(y)}\Vert x-y\Vert_2^2}}  \\
    &= \sqrt{\frac{\frac{1}{|X|}\sum_{i=1}^k \frac{1}{|C_i|-1}\sum_{\substack{x\neq y \\x,y\in C_i}}\Vert x-y\Vert_2^2}{\frac{1}{|X|^2}\sum_{x,y\in X}\Vert x-y\Vert_2^2}} \\
    &\underset{\Cref{clm_norm_squared}}{=} \sqrt{\frac{\sum_{i=1}^k \frac{1}{|C_i|-1} |C_i|\sum_{\substack{x\in C_i}} \Vert x-\mu^{(i)} \Vert_2^2 }{\frac{1}{|X|}\sum_{x,y\in X}\Vert x-y \Vert_2^2}} \\
    &\leq \sqrt{\frac{ 2\cdot \sum_{i=1}^k \sum_{x\in C_i} \Vert x- \mu^{(i)} \Vert_2^2}{\frac{1}{|X|}\sum_{x,y\in X}\Vert x-y \Vert_2^2}} \\
    &= \sqrt{2\cdot\frac{\text{cost}(\mathcal{C})}{\frac{1}{|X|}\sum_{x,y\in X}\Vert x-y \Vert_2^2}} .
\end{align*}
\end{proof}
Note that the denominator term $\frac{1}{|X|}\sum_{x,y\in X}\Vert x-y \Vert_2^2$ is a fixed size for the dataset and is independent of the reference clustering $\mathcal{C}$. 
Thus, with an appropriate choice of graph to describe the reference clustering, \Cref{thm:geomcheeger} yields coordinate cuts with conductance directly related to the reference clustering cost.

\subsection{IMM Alternative Analysis}\label{sec:immproof}
In this section we focus on $k$-medians clustering in the $\ell_1$ norm. For every subset $C\subset\R^d$, denote its median by
\[
\mathrm{med}(C) = \min_{z\in\R^d}\sum_{x\in C}\norm{x-y}_1 .
\]
Let $\mathcal{C}=C_1,\ldots,C_k$ be a reference $k$-medians clustering of $X$ with respective cluster centroids $\mu^{(i)}=\mathrm{med}(C_i)$. The $k$-medians clustering cost of $\mathcal{C}$ is defined as
\[
  \mathrm{cost}_1(\mathcal{C}) = \sum_{i=1}^k\sum_{x\in C_i}\norm{x-\mu^{(i)}}_1 .
\]
Let $T_{\mathrm{IMM}}(\mathcal{C})$ be the explainable clustering tree returned by IMM algorithm. \cite{moshkovitz2020explainable} proved the following theorem:
\begin{theorem}\label{thm:imm}
    For every reference clustering $\mathcal{C}$ with centroids, $\mathrm{cost}_1(T_{\mathrm{IMM}}(\mathcal{C})) \leq O(k) \cdot \mathrm{cost}_1(\mathcal{C})$.
\end{theorem}
In the terminology from~\Cref{sec:introduction}, the ``price of explainability'' of $k$-medians clustering is $O(k)$. This remains the best bound to date for a deterministic algorithm.\footnote{As mentioned in~\Cref{sec:related}, for randomized algorithms, a tight bound of $(1+o(1))\log k$ is known.}

In this section we prove the same theorem  for our slightly modified variant from \Cref{sec:imm}, where it was cast in terms of minimizing the non-uniform cut sparsity. The goal is demonstrate how the graph partitioning framework can be used analytically, and present an alternative and arguably simpler proof for the same price of explainability upper bound.
Let $\tilde T_{\mathrm{IMM}}(\mathcal{C})$ denote the tree returned by our modified IMM.  

\begin{theorem}\label{thm:immtag}
    For every reference clustering $\mathcal{C}$ with centroids, $\mathrm{cost}_1(\tilde T_{\mathrm{IMM}}(\mathcal{C})) \leq O(k) \cdot \mathrm{cost}_1(\mathcal{C})$.
\end{theorem}
For the proof, we introduce some notation aligned with \cite{moshkovitz2020explainable} for better comparison and readability.
Recall that the IMM algorithm builds an explainable decision tree on both the points $X$ and the centroids $M=\{\mu^{(1)},\ldots,\mu^{(k)}\}$, such that each tree leaf contains exactly one centroid, and partition of $X$ into the $k$ leaves forms the explainable clustering.
%
%
For every $x\in X$, denote by $\mu(x)\in M$ its centroid in the reference clustering $\mathcal{C}$. 
For every tree node $u$, let $X^u\subset X$ and $M^u\subset M$ be the subsets of points and centroids, respectively, contained in $u$, and let $Y^u=X^u\cup M^u$. A \emph{mistake} in $u$ is a pair $x\in X^u$, $c(x)\in M^u$ that are separated by the threshold coordinate cut in $u$. Let $t^u$ denote the number of mistakes in $u$. 
%
%
Also, recall the in our modified IMM, in each node $u$ we fix a pair of centroids $\mu_u',\mu_u''$ at maximal distance among the centroids in $u$:
\[
  \norm{\mu_u'-\mu_u''}_1 = \min_{i,j:\tilde\mu',\tilde\mu''\in M^u}\norm{\tilde\mu'-\tilde\mu''}_1 .
\]
The proof of \Cref{thm:immtag} has two steps, analogous to the proof of \Cref{thm:imm}. The first step is a straightforward consequence of the triangle inequality. The detailed proof appears in \cite{moshkovitz2020explainable} and is omitted here.
\begin{lemma}[Lemma 5.5 in \cite{moshkovitz2020explainable}]\label{lmm:imm1}
    \[ \mathrm{cost}_1(\tilde T_{\mathrm{IMM}}(\mathcal C)) \leq \mathrm{cost}_1(\mathcal C) + \sum_{u\in T_{\mathrm{IMM}}'(\mathcal C)}t_u\cdot\norm{\mu_u'-\mu_u''}_1 . \]
\end{lemma}
The second step is the more involved part of the proof, and the part where the graph-based analysis presented here departs from \cite{moshkovitz2020explainable}. 
\begin{lemma}[Lemma 5.6 in \cite{moshkovitz2020explainable}]\label{lmm:imm2}
    Let $H$ be the height of $\tilde T_{\mathrm{IMM}}(\mathcal C)$. Then,
        \[   \sum_{u\in T_{\mathrm{IMM}}'(\mathcal C)}t_u\cdot\norm{\mu_u'-\mu_u''}_1 \leq H\cdot \mathrm{cost}_1(\tilde T_{\mathrm{IMM}}(\mathcal C)). \]
\end{lemma}
Since $\tilde T_{\mathrm{IMM}}(\mathcal C)$ has $k$ leaves, its height is at most $k$, thus \Cref{lmm:imm1,lmm:imm2} together immediately imply \Cref{thm:immtag}. 

\begin{proof}[Proof of \Cref{lmm:imm2}]
Fix a node $u$ in $\tilde T_{\mathrm{IMM}}(\mathcal C)$. 
Let $\times_{j=1}^d[a_j^u,b_j^u]$ be a bounding box containing all of $Y^u$. 
Define the following distribution $\mathcal{D}$ over coordinate cuts $j,\tau$: first, pick $j\in\{1,\ldots,d\}$ with probability proportional to $|b_j^u-a_j^u|$; then, pick $\tau$ uniformly at random in $[a_j^u,b_j^u]$. 

Let $\tilde G(Y^u,\tilde E)$ be any undirected graph over $Y^u$. 
The probability that an edge $xy\in \tilde E$ is cut by $(j,\tau)\sim\mathcal{D}$ is
\[
  \sum_{j=1}^d\frac{|b_j^u-a_j^u|}{\norm{b^u-a^u}_1}\cdot \frac{|x_j-y_j|}{|b_j^u-a_j^u|} = \frac{\norm{x-y}_1}{\norm{b^u-a^u}_1},
\]
and therefore,
\begin{equation}\label{eq:edgecut}
    \E_{(j,\tau)\sim\mathcal{D}}\left[e_{\tilde G}(S_{j,\tau},Y_u\setminus S_{j,\tau})\right] = \frac{1}{\norm{b^u-a^u}_1}\sum_{xy\in\tilde E}\norm{x-y}_1 .
\end{equation}
Let $\mathcal{S}_u$ be the set of coordinate cuts $(j,\tau)$ that satisfy $\tau\in[a_j,b_j]$. Let $m_u(j,\tau)$ denote the number of mistakes that a cut $j,\tau$ makes (i.e., the number of points in $x\in X^u$ such that $c(x)\in M^u$, but the cut separates $x$ and $c(x)$).  
Recall that our modified IMM chooses the cut $j_u,\tau_u$ in the node among the cuts in $\mathcal{S}_u$ that separate the pair $\mu_u',\mu_u''$, and $t_u=m_u(j_u,\tau_u)$. 

Let $G(Y^u,E_G)$ be the star graph on $Y^u$, where each $x\in X^u$ is adjacent to $c(x)$ if $c(x)\in M^u$, and is an isolated node otherwise. Observe that $m_u(j,\tau)=e_{G}(S_{j,\tau},Y_u\setminus S_{j,\tau})$. Let $H(Y^u,E_H)$ be the graph that contains an single edge between $\mu_u',\mu_u''$. Observe that $e_{H}(S_{j,\tau},Y_u\setminus S_{j,\tau})$ equals $1$ if the cut $j,\tau$ separates $\mu_u',\mu_u''$, and $0$ otherwise. We denote this as $e_{H}(S_{j,\tau},Y_u\setminus S_{j,\tau}) = \mathbf1\{\text{$(j,\tau)$ separates $\mu_u',\mu_u''$}\}$. 
Therefore,
\begin{align*}
    t_u = m_u(j_u,\tau_u) &= \min_{\substack{j,\tau\in\mathcal{S}_u:\\
    \text{$(j,\tau)$ separates $\mu_u',\mu_u''$}}}m_u(j,\tau) \\
    &= \min_{j,\tau\in\mathcal{S}_u}\frac{m_u(j,\tau)}{\mathbf1\{\text{$(j,\tau)$ separates $\mu_u',\mu_u''$}\}} & \\
    &= \min_{j,\tau\in\mathcal{S}_u}\frac{e_G(S_{j,\tau},Y_u\setminus S_{j,\tau})}{e_H(S_{j,\tau},Y_u\setminus S_{j,\tau})} & \\
&\leq \frac{\E_{(j,\tau)\sim\mathcal{D}}\left[e_{G}(S_{j,\tau},Y_u\setminus S_{j,\tau})\right]}{\E_{(j,\tau)\sim\mathcal{D}}\left[e_{H}(S_{j,\tau},Y_u\setminus S_{j,\tau})\right]} & \text{\Cref{clm:dansfav}} \\
&= \frac{\sum_{xy\in E_G}\norm{x-y}_1}{\sum_{xy\in E_H}\norm{x-y}_1} & \text{\Cref{eq:edgecut}} \\
&= \frac{\sum_{x\in X^u}\norm{x-c(x)}_1}{\norm{\mu_u'-\mu_u''}_1} & \text{definition of $G$ and $H$.} 
\end{align*}
Rearranging, $t_u\cdot\norm{\mu_u'-\mu_u''}_1\leq\sum_{x\in X^u}\norm{x-c(x)}_1$. Since in each level $L$ in the tree the clusters $\{X^u:u\in L\}$ form a partition of $X$, we get
\[
  \sum_{u\in L}t_u\cdot\norm{\mu_u'-\mu_u''}_1 \leq \sum_{u\in L}\sum_{x\in X^u}\norm{x-c(x)}_1 = \sum_{x\in X}\norm{x-c(x)}_1 = \mathrm{cost}_1(\mathcal{C}).
\]
Summing again over the $H$ levels in the tree yields the lemma. 
\end{proof}
Note that the ratio $\frac{e_G(S_{j,\tau},Y_u\setminus S_{j,\tau})}{e_H(S_{j,\tau},Y_u\setminus S_{j,\tau})}$ that arises in the proof is the non-uniform cut sparsity $\Psi_{G,H}(S_{j,\tau})$ from \Cref{sec:nonuniform}.

\section{Additional Experimental Results}\label{app:experiments}

\paragraph{Running times.}
\Cref{tab:runtime} contains running time measurements. For reference-based methods (EMN, CART and~\clique) we measure the reference clustering step and tree construction step separately. For~\knn, we measure the $k$-NN graph construction step and the tree construction step separately. 

In our experiments, Kernel IMM proved feasible to run only on the smaller datasets (R15, Pathbased, Iris, Ecoli, Breast Cancer). Each algorithm was allotted three hours to run on each dataset. On the larger datasets, Kernel IMM either ran out of memory, or failed to complete running within the allotted time. All other methods finished running within up to 16 minutes (see \Cref{tab:runtime}). Therefore, we report results for Kernel IMM only for the smaller datasets. 

\paragraph{kNN graph parameters.}
\Cref{tab:k-tuning} includes additional results for~\knn, showing how its performance changes as $k$ (the parameter of the $k$-NN graph) varies on the Ecoli, Iris and Breast Cancer datasets.

\paragraph{Number of leaves.}
As mentioned in \Cref{sec:prelim}, \spex~can produce a tree with any desired number of leaves $\ell$. Our evaluation so far has focused on the setting $\ell=k$, i.e., the number of leaves in the output tree is equal to the number of clusters in the reference clustering. 
To clarify the relation between $k$ and $\ell$, viewing the reference clustering as a trained model, $k$ is its number of outcomes and $\ell$ is the number of explanations our explainability method can yield. Since each outcome must have its own separate explanation, we must have $\ell\geq k$, thus $\ell=k$ is the ``most explainable'' setting. As $\ell$ increases, the expressiveness of the explainable clustering tree grows and it is better able to approximate the reference clustering, albeit at the cost of explainability, since now some model outcomes would have multiple different explanations. 

\Cref{tab:moreleaves} includes result with $\ell=2k$ for \clique, CART-Spectral, CART-$k$-means, and ExKMC \cite{frost2020exkmc}, which is an extension of IMM to support more leaves (which IMM does not naturally support, due to its reliance on reference clustering centroids; the same goes for EMN). The results do not point to a clearly superior method.


\begin{table*}[t]
\vspace{0.1in}
{\fontsize{9}{11}\selectfont  
\centering
\caption{Runtime comparison of the non kernel k-means based algorithms. Presented times are the median runtime across five runs. For the Clique and CART algorithms, performance was measured based on a spectral reference.}
\label{tab:runtime}
}
\vspace{0.1in}
\centering
\resizebox{\textwidth}{!}{\setlength{\tabcolsep}{2.2pt} 
\begin{tabular}{l *{10}{r}}
\toprule
Algorithm & Beans & Iris & CIFAR & Caltech 101 & R15 & Pathbased & Ecoli & Cancer & Newsgroups & MNIST\\
\midrule
Spectral reference & 1.8s & 22.1ms & 1m12s & 7.6s & 91.6ms & 29ms & 34.2ms & 34.1ms & 8.54s & 1min53s\\
k-means reference & 86.4ms & 4.21ms & 2.69s & 2.1s & 10.4ms & 4.16ms & 5.39ms & 7.01ms & 1.04s & 5.16s\\
$k$-NN graph build & 1.87s & 7.01ms & 1m55s & 2.23s & 19ms & 3.57ms & 9.31ms & 8.03ms & 5.79s & 2min59s\\
\midrule
\knn & 3.69s & 27ms & 5m53s & 1m48s & 93.4ms & 22.8ms & 85.3ms & 90.8ms & 2min27s & 2min52s\\
\clique & 203ms & 8.36ms & 28.2s & 1m12s & 17ms & 9.65ms & 13.3ms & 16.3ms & 13.2s & 39.4s\\
EMN & 13.5s & 6.51ms & 5m13s & 15m40s & 101ms & 13.3ms & 20.9ms & 184ms & 7min56s & 5min20s\\
CART & 349ms & 14.8ms & 32.2s & 8min51s & 72.7ms & 9.4ms & 33.9ms & 31.9ms & 39.2s & 30s\\
\bottomrule
\end{tabular}
}
\end{table*}


\begin{table*}[t]
\vspace{0.1in}
{\fontsize{9}{11}\selectfont  
\centering
\caption{Performance of~\knn~for different $k$ values. The reference compared in the REF column is the reference with the same $k$ applied.}
\label{tab:k-tuning}
}
\vspace{0.1in}
\centering
\footnotesize
\begin{tabular}{l *{3}{ccc}}
\toprule
& \multicolumn{3}{c}{Ecoli} & \multicolumn{3}{c}{Iris} & \multicolumn{3}{c}{Cancer}\\
\cmidrule(lr){2-4} \cmidrule(lr){5-7} \cmidrule(lr){8-10}
$k$ & ARI & AMI & REF & RS & AMI & REF & ARI & AMI & REF\\
\midrule
2
& .594 & .571 & .006 &
.287 & .370 & .020 &
.681 & .603 & -.019 \\
5
& .594 & .589 & .640 &
.450 & .647 & .895 &
.594 & .546 & .688 \\
10
& .682 & .648 & .828 &
.450 & .647 & .514 &
.507 & .490 & .547 \\
15
& .682 & .648 & .854 &
.450 & .647 & .445 &
.507 & .490 & .567 \\
20
& .679 & .642 & .863 &
.450 & .647 & .450 &
.507 & .490 & .562 \\
50
& .679 & .638 & .593 &
.600 & .642 & .756 &
.507 & .490 & .532 \\
\bottomrule
\end{tabular}
\end{table*}

\vspace{0.4in}
\begin{table*}[h]
\vspace{0.1in}
{\fontsize{9}{11}\selectfont  
\centering
\caption{Results with $\ell=2k$ leaves.}
\label{tab:moreleaves}
}
\vspace{0.1in}
\centering
\resizebox{\textwidth}{!}{
\begin{tabular}{l *{8}{cc}}
\toprule
& \multicolumn{2}{c}{Ecoli} & \multicolumn{2}{c}{Iris} & \multicolumn{2}{c}{Cancer} & \multicolumn{2}{c}{MNIST} & \multicolumn{2}{c}{Caltech 101} &
\multicolumn{2}{c}{Newsgroups} & \multicolumn{2}{c}{Beans} & \multicolumn{2}{c}{Cifar}\\
\cmidrule(lr){2-3} \cmidrule(lr){4-5} \cmidrule(lr){6-7} \cmidrule(lr){8-9} \cmidrule(lr){10-11} \cmidrule(lr){12-13} \cmidrule(lr){14-15} \cmidrule(lr){16-17}
Algorithm & ARI & AMI & ARI & AMI & ARI & AMI & ARI & AMI & ARI & AMI & ARI & AMI & ARI & AMI & ARI & AMI \\
\midrule
\clique & .469 & .589 & .716 & .739 & .491 & .464 & .197 & .317 & .203 & .565 & .088 & .281 & .370 & .512 & .309 & .484\\
ExKMC & .458 & .582 & .730 & .755 & .491  & .464  & .224  &  .328   & .249   & .546 & .096 & .263  & .370 & .512 & .323 & .467\\
CART-$k$-means & .447  & .573 & .716 & .739 & .491 & .464 & .031 & .177 & 0  & .208 & .026 & .148 & .370 & .512 & .144 & .371\\
CART-Spectral & .434 & .493 & .610 & .648 & .749 & .631 & .080 & .210 & 0 & .214 & .006 & .104 & .608 & .695 & 135 & .381\\
\bottomrule
\end{tabular}
}
\end{table*}

\end{document}